\documentclass[11pt,reqno]{amsart}
\usepackage[top=1.2in, bottom=1.2in, left=1.2in, right=1.2in]{geometry}
\usepackage{amsfonts, amssymb, amsmath, dsfont, mathtools}
\usepackage{mathrsfs} 
\usepackage[linktoc=page, colorlinks, linkcolor=blue, citecolor=blue]{hyperref}
\usepackage{enumerate}

\usepackage{graphicx}
\usepackage{xcolor}
\usepackage{subcaption} 
\usepackage[font=small]{caption} 
\numberwithin{equation}{section}

\DeclareMathOperator{\im}{Im}
\DeclareMathOperator{\id}{id}

\newcommand{\ip}[2]{\langle#1,#2\rangle}

\def \R {\mathbb{R}}

\def \OO {\mathcal{O}}

\def \a {\alpha}

\def \s {\sigma}

\def \one {{\textbf 1}}

\newcommand{\blue}{}

\newtheorem{theorem}{Theorem}[section]
\newtheorem{proposition}[theorem]{Proposition}
\newtheorem{corollary}[theorem]{Corollary}
\newtheorem{lemma}[theorem]{Lemma}

\newtheorem{definition}[theorem]{Definition}

\theoremstyle{remark}
\newtheorem{remark}[theorem]{Remark}
\newtheorem{example}[theorem]{Example}

\begin{document}

\title{The Capacity of feedforward neural networks}
\author{Pierre Baldi \and Roman Vershynin}
\date{\today}

\address{Department of Computer Science, University of California, Irvine}
\email{pfbaldi@uci.edu}

\address{Department of Mathematics, University of California, Irvine}
\email{rvershyn@uci.edu}

\thanks{Work in part supported by DARPA grant D17AP00002
and NSF grant 1839429 to P. B., and U.S. Air Force grant FA9550-18-1-0031 to R. V}

\date{\today}

\begin{abstract}
A long standing open problem in the theory of neural networks is the development of quantitative methods to estimate and compare the capabilities of different architectures. Here we define the capacity of an architecture by the binary logarithm of the number of functions it can compute, as the synaptic weights are varied. 
 {The capacity provides an upperbound on the number of bits that can be extracted from the training data and stored in the architecture during learning.} We study the capacity of layered, fully-connected, architectures of linear threshold neurons with $L$ layers of size $n_1,n_2, \ldots, n_L$ and show that in essence the capacity is given by a cubic polynomial in the layer sizes:
$C(n_1,\ldots, n_L)=\sum_{k=1}^{L-1} \min(n_1,\ldots,n_k)n_kn_{k+1}$,
{where layers that are smaller than all previous layers act as bottlenecks.}
In proving the main result, we also develop new techniques (multiplexing, enrichment, and stacking) as well as new bounds on the capacity of finite sets.
We use the main result to identify architectures with maximal or minimal capacity under a number of natural constraints. This leads to the notion of structural regularization for deep architectures. While in general, everything else being equal, shallow networks compute more functions than deep networks, the functions computed by deep networks are more regular and ``interesting''. 
\end{abstract}

\maketitle

\setcounter{tocdepth}{1}
\tableofcontents

\section{Introduction}

Since their early beginnings (e.g. \cite{mcculloch:43,rosenblatt1958perceptron}),
neural networks have come a significant way. {Today 
they are at the center of myriads of successful applications, spanning the gamut from games all the way to biomedicine
\cite{schmidhuber2015deep,silver2017mastering1,baldireview2018}.} In spite of these successes, the problem of quantifying the power of a neural architecture, in terms of the space of functions it can implement as its synaptic weights are varied, has remained open. {This quantification is fundamental to the science of neural networks. It is also important for applications in order to compare architectures, including the basic comparison between deep and shallow architectures, and to select the most efficient architectures. }
Furthermore, this quantification is essential for understanding the apparently unreasonable properties of deep learning and the well known paradox that deep learning architectures have a tendency to {\it not} overfit, even when the number of synaptic weights significantly exceeds the number of training examples \cite{zhang2016understanding}. 
To address these problems, in this work we introduce a notion of capacity for neural architectures and study how this capacity can be computed. {We focus primarily on architectures that are feedforward, layered, and fully connected  denoted by: $A(a_1,n_2, \ldots,n_L)$, where $n_i$ is the number of neurons in layer $i$.}

\subsection{{Functional capacity of neural networks}}
{
Ideally, one would like to be able to describe the {\em functional capacity} of a neural network architecture, i.e. completely characterize the class of functions that it can compute as its synaptic weights are varied. In the purely linear case,  such a program can easily be carried out. Indeed, let $p=\min(n_2,\ldots,n_{L-1})$. If $p \geq n_1$,
then  $A(n_1,\ldots,n_L)$ is simply the class of all linear functions from 
$\R^{n_1}$ to $\R^{n_L}$, i.e. it is equivalent to $A(n_1,n_L)$. If $p < n_1$,
then there is a rank restriction and $A(n_1,\ldots,n_L)$ is the class of all linear functions from $\R^{n_1}$ to $\R^{n_L}$ with rank less or equal to $p$, i.e. it is equivalent to $A(n_1,p,n_L)$ \cite{baldi89}. In addition, if $n_L \leq p$ then the effect of the bottleneck layer is nullified and 
$A(n_1,p,n_L)$ is equivalent to 
$A(n_1,n_L)$, i.e. the effect of the bottleneck restriction is nullified.
The exact same result is true in the case of unrestricted Boolean architectures (i.e. architectures with no restrictions on the Boolean functions being used), where the notion of rank is replaced by the fact that a Boolean layer of size $p$ can only take $2^p$ distinct values.}

{
Unfortunately, in the other and most relevant non-linear settings, such a program has proven to be difficult to carry out, except for some important, but limited, cases. 
Indeed, for a single threshold gate neuron, $A(n,1)$ corresponds to the set of linearly separable functions. Variations of this model using sigmoidal or other non-linear transfer functions can be understood similarly. Furthermore, in the case of an 
$A(n_1,n_2)$ architecture, for a given input, the output of each neuron is independent of the weights, or the outputs, of the other neurons. Thus the functional capacity of 
$A(n_1,n_2)$ can be described in terms of $n_2$ independent $A(n_1,1)$ components.
When a single hidden layer is introduced, the main known results are those of universal approximation properties. In the Boolean case, using linear threshold gates, 
and noting that these gates can easily implement the standard AND, OR, and NOT Boolean operations, it is easy to see using conjunctive or disjunctive normal form that 
$A(n_1,2^{n_1},1)$ can implement any Boolean function of $n_1$ variables, and thus
$A(n_1,2^{n_1},m)$ can implement any Boolean map from $\{0,1\}^{n_1}$ to $\{0,1\}^{m}$.
It is also known that in the case of Boolean unrestricted autoencoders, the corresponding architectures $A(n_1,n_2,n_1)$ implement clustering \cite{baldiboolean12,baldijmlr12}.
In the continuous case, there are various universal approximation theorems 
\cite{hornik1989,cybenko1989approximation}
showing, for instance, that continuous functions defined over compact sets can be approximated to arbitrary degrees of precision by architectures of the form $A(n_1,\infty,m)$, where we use ``$\infty$'' to denote the fact that the hidden layer may be arbitrary large. Beyond these results, very little is known about the functional capacity of 
$A(n_1, \ldots, n_L)$. }
 
\subsection{{Cardinal capacity of neural networks}}
{
In order to make progress on the capacity issue, here we define a simpler notion of capacity, the cardinal capacity. The {\em cardinal capacity} $C(A)$ of a finite class $A$ of functions is simply the logarithms base two of the number of functions contained in $A$ (Figure \ref{fig:Framework}): $C(A) =\log_2 \vert A \vert$. The cardinal capacity can thus be viewed as the number of bits required to specify, or communicate, an element of $A$, in the worst case of a uniform distribution over $A$. We are particularly interested in computing the cardinal capacity of feedforward architectures $A(n_1,\ldots,n_L)$ of linear threshold functions. In continuous settings, the cardinal capacity can be defined in a similar way in a measure theoretic sense by taking the logarithm of the volume associated with $A$. In the rest of the paper, in the absence of any qualification, the term capacity is used to mean cardinal capacity.}

{
While in general the notion of cardinal capacity is simpler and less informative than the notion of functional capacity, in the case of neural architectures the cardinal capacity has a very important interpretation. Namely, it provides an upperbound on the number of bits that can be stored in the architecture during learning. Indeed, the learning process can be viewed as a communication process over the learning channel \cite{baldi2016local}, whereby information is transferred from the training data to the synaptic weights. Thus the learning process is a process for selecting and storing an element of $A$, which corresponds to $C(A)$ bits of information. Any increase in the precision of the synaptic weights that does not change the input-output function is not visible from the outside.}

\begin{figure}[ht!]
  \centering
  \includegraphics[width=.60\textwidth]{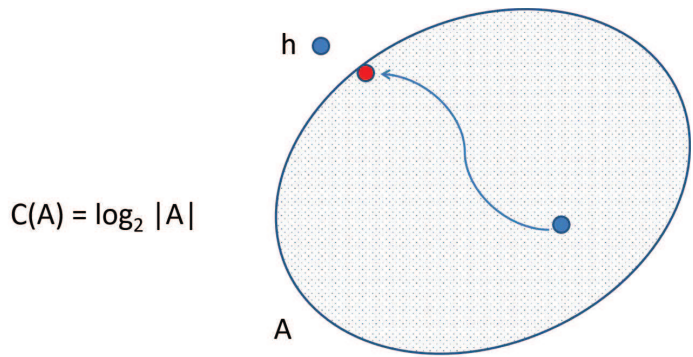}
\caption{{Learning framework where $h$ is the function to be learnt and $A$ is the available class of hypothesis or approximating functions. The cardinal capacity is the logarithm base two of the number, or volume, of the functions contained in $A$.}}  
\label{fig:Framework}
\end{figure}

The bulk of this paper focuses on estimating the capacity of arbitrary feedforward, layered and fully-connected, architectures of any depth which are widely used in many applications. {As a side note, the capacity of fully connected recurrent networks is studied in \cite{baldi2018neuronal}.}
In the process, several techniques and theorems of self-standing interest are developed. In addition, the extremal properties of the capacity of such architectures is analyzed, contrasting 
the capacity of shallow versus deep architectures, and leading to the notion of structural regularization. Structural regularization provides 
a partial explanation for why deep neural networks have a tendency to avoid 
overfitting.

\subsection{Main result of the paper: the capacity formula}
The main result of this paper, Theorem~\ref{thm: main}, provides an estimate of the capacity of a general feedforward, layered, fully connected neural network of linear threshold gates. Suppose that such network has 
$L$ layers with $n_k$ neurons in layer $k$, where $k=1$ corresponds to the input layer and $n_L$ correspond to the output layer.
We show that, under some very mild assumptions on the sizes of the layers, 
the capacity $C(n_1,\ldots,n_L)$ of this network, defined as the binary logarithm 
of the total number of functions $f : \{0,1\}^{n_1} \to \{0,1\}^{n_L}$ it can compute, satisfies 
\begin{equation}	\label{eq: capacity intro}
C(n_1,\ldots,n_L) \asymp \sum_{k=1}^{L-1}\min(n_1,\ldots,n_k)n_k n_{k+1}.
\end{equation}

Here the notation $a \asymp b$ means that there exists two positive absolute constants
$c_1, c_2$ such that $c_1 b \le a \le c_2 b$. Actually, we will show that the upper bound 
in the capacity formula \eqref{eq: capacity intro} holds with constant $c_2 = 1$. The absolute constant 
$c_1 \in (0,1)$ hidden in the lower bound may not depend on anything, 
in particular it is independent of the depth $L$ of the network, or the widths $n_k$ of the layers. 
The formula \eqref{eq: capacity intro} thus shows that the capacity of such a network is essentially given by a 
cubic polynomial in the sizes of the layers, 
where the bottleneck layers play a special role.

\subsection{Capacity of sets}
In the process of proving the main capacity formula \eqref{eq: capacity intro}
we establish some other stand alone results of independent value. 
At the heart of our analysis are new lower bounds on the {\em capacity of sets}. 
We define the capacity $C(S)$ of a set $S \subset \R^n$ as the binary logarithm
of the number of all the linear threshold functions $f: S \to \{0,1\}$. 
In other words, $C(S)$ measures the capacity of a single neuron when the inputs are restricted to $S$. 
Equivalently, $C(S)$ is the binary logarithm of the number of all possible ways $S$ can be separated by affine hyperplanes. 
We prove that for any subset $S$ of the Boolean cube $\{0,1\}^n$, the capacity satisfies
$$ 
\frac{1}{16} \log_2^2 \vert S\vert \leq C(S) 
\leq 1 + n \log_2 \left ( \frac{e\vert S \vert}{n} \right ).
$$
The upper bound was previously known, it holds for any subset $S \subset \R^n$, 
and it can be replaced by the simpler form $n \log_2 \vert S\vert$ when $n \geq 4$. 
The lower bound is a new contribution, and it
improves over the previously known (and easy) lower bound of 
$1+ \log_2 \vert S\vert$, which is also true for any set $S \subset \R^n$, as soon as $\vert S\vert> 2^{17}$.
\medskip

In the next section, we provide mathematical definitions of feedforward
neural networks and their capacities and describe several known results. 
A reader familiar with neural network theory may glance through it and rapidly go to Section~\ref{s: new results}, which provides a 
description of the new results and provides a roadmap for the paper.

\section{Neural architectures and their capacities}
\label{s: notation}

\subsection{Threshold functions and maps}

Throughout this paper, the $n$-dimensional {\em Boolean cube} is denoted by:
$$
H^n =\{ 0,1 \}^n.
$$
The {\em Heaviside function} $h: \R \to \{0,1\}$ is defined by:
$$
h(t) =
\begin{cases}
  1, & t \ge 0 \\
  0, & t < 0.
\end{cases}
$$
We have chosen the $\{0,1\}$ formalism for convenience.
It can easily be replaced with the $\{-1,1\}$ formalism, using
the Boolean cube $\{-1,1\}^n$ and replacing the Heaviside function $h$ 
by the sign function. 

\begin{definition}[Threshold functions]
  Consider a set $S \subset \R^n$.
  A function $f : S \to \{0,1\}$ is called a {\em (linear) threshold function} on $S$
 if there exist $a \in \R^n$ and $\alpha \in \R$ such that $f$ can be expressed as:
  $$
  f(x) 
  = h \big( \ip{a}{x} + \alpha \big), \quad x \in S.
  $$
  The set of all threshold functions on $S$ is denoted by $T(S,n,1)$ 
  and is often abbreviated to $T(S)$. 
\end{definition}

%
The notion of threshold functions generalizes naturally to the multivariate setting. 

\begin{definition}[Threshold maps]
  Consider a set $S \subset \R^n$.
  A function $f = (f_1,\ldots, f_m) : S \to H^m$ is called a {\em threshold map} 
  if all components $f_i$ are threshold functions on $S$.
  The set of all threshold maps is denoted $T(S,n,m)$;
  in the particular case where $S = H^n$, we abbreviate it to $T(n,m)$.
\end{definition}

\subsection{Neural architectures}

A neural architecture (or network) is represented by a weighted directed graph, 
where the nodes represent neurons and the weights represent synaptic connection strengths. Neurons have numerical states and operate by taking the weighted average 
of the corresponding parent states and applying a transfer function to this weighted average. This paper focuses on one of the most widely used class of architectures, 
namely {\em layered feedforward neural architectures}, where neurons are arranged into layers and connections run from one layer to the next. We will denote an architecture
with $L$ layers numbered from 1 to $L$, and $n_k$ neurons in each layer $k$, by:
$$
A(n_1,n_2,\ldots, n_L).
$$
The architecture has $n_1$ input neurons and $n_{L}$ output neurons. 
To simplify the analysis, we make two assumptions:

\begin{enumerate}[\quad (a)]
  \item full connectivity between the layers, i.e. we assume that each neuron in layer $k$ is connected to every neuron in layer $k+1$, but not to any other neurons;
      \item the transfer function is the Heaviside threshold function.  
\end{enumerate}
These two assumptions are not absolutely essential, and we discuss how to relax them in the conclusion.

Under these assumptions, the input-output function computed by a layered feedforward neural architecture with a fixed set of weights is a composition of threshold maps. Indeed, the architecture $A(n_1,n_2,\ldots, n_L)$ computes an input-output function of the form 
\begin{equation}	\label{eq: composition}
f = f_{L-1} \circ \cdots \circ f_2 \circ f_1
\end{equation}
where each $f_k: \R^{n_k} \to H^{n_{k+1}}$ is a threshold map.
Generally, a network architecture $A(n_1,\ldots, n_L)$
is able to compute infinitely many functions
$f: \R^{n_1} \to H^{n_L}$, as the synaptic weights and biases (threshold values) are varied. However, if the inputs are restricted to a given finite set $S \subset \R^{n_1}$, 
then the number functions $f: S \to H^{n_L}$ computable by the architecture becomes finite. 
We denote this class of functions by
$$
T(S, n_1,\ldots,n_L).
$$
In the most important case where $S = H^{n_1}$ is the Boolean cube, 
we drop the set $S$ from the notation. Thus, 
$$
T(n_1,\ldots,n_L)
$$
denotes the class of functions $f : H^{n_1} \to H^{n_L}$ 
computable by the architecture $A(n_1,\ldots, n_L)$;
it consists of all functions that can be expressed as in \eqref{eq: composition}
for some set of threshold maps $f_k \in T(n_k, n_{k+1})$ ($k=1, \ldots, L-1$).

\subsection{Definition of capacity}		\label{s: def capacity}

The main question we address in this paper is: how many different functions can a given neural architecture compute? This leads us to the notion of capacity, which we define as follows:

\begin{definition}[Capacity of a neural architecture]
  The capacity of a neural architecture $A(n_1,n_2,\ldots, n_L)$ is the binary logarithm of 
  the number of different functions $f : H^{n_1} \to H^{n_L}$ it can compute, i.e.
  $$
  C(n_1,\ldots,n_L) = \log_2 \vert T(n_1, \ldots, n_L) \vert.
  $$
  More generally, the capacity of a neural architecture on a given finite set $S \subset \R^{n_1}$ is 
  $$
  C(S,n_1,\ldots,n_L) = \log_2 \vert T(S, n_1, \ldots, n_L) \vert.
  $$
\end{definition}

The capacity of an architecture can be interpreted as the number of bits required to specify a function computable by the architecture. Remarkably, this can also be viewed as an {\it upper bound on the number of bits that can stored in the architecture during learning}, or equivalently an upper bound on the number of bits that can be extracted from the training data. If the capacity of a network is $C$ and the number of connections (weights) is $W$, then at most $C/W$ bits can be stored on average per synaptic weight. This bound is independent, and more fundamental, than any hardware limitation on the precision of the synaptic weights. It can be viewed as a bound on the effective capacity of the deep learning channel \cite{baldi2016local}. The bound holds even if the weights have infinite precision and thus in principle contain an infinite amount of information. This is because in the current framework different sets of weights that implement the same overall input-output function are indistinguishable from the outside world.

\medskip

Getting optimal bounds on the capacity can be non-trivial even for simple network architectures.
Consider, for example, a {\em single-neuron network} $A(n,1)$ with $n$ inputs, which thus can implement any threshold function on $\R^n$. The capacity $T(S,n,1)$ of this network on a given set of inputs $S \subset \R^n$ is the logarithm of the number of all distinct threshold functions that can be defined on $S$. We call this quantity the capacity of $S$. 

\begin{definition}[Capacity of a set]
  The capacity of a set $S \subset \R^n$ is the binary logarithm 
  of the number of all threshold functions on $S$, i.e. 
  $$
  C(S) = \log_2 |T(S)|. 
  $$
  Equivalently, $C(S)$ is the binary logarithm of the total number of 
  ways the set $S$ can be partitioned by affine hyperplanes in $\R^n$ (taking into account the binary assignment associated with each partition). 
\end{definition}

A significant part of this paper is devoted to studying the capacity of sets, in particular to 
deriving optimal estimates for $C(S)$ in terms of the cardinality of $S$.

\subsection{Basic properties of capacity}

Here we summarize a few elementary properties of the capacity of neural architectures.

\begin{lemma}[Basic properties of capacity] 		\label{lem: capacity basic}
  \quad
  \begin{enumerate}[\quad 1. ]
    \item \label{pr: affine invariance}
      (Affine invariance) For any invertible affine transformation $F: \R^{n_1} \to \R^{n_1}$,
      we have:
      $$
      C \big( F(S),n_1,n_2,\ldots,n_L \big) = C(S,n_1,n_2,\ldots,n_L).
      $$
    \item \label{pr: monotonicity}
      (Monotonicity) If $n_k \le m_k$ for all $k$, then: 
      $$
      C(n_1,\ldots,n_L) \le C(m_1,\ldots,m_L).
      $$
    \item \label{pr: sub-additivity}
      (Sub-additivity) For any $1 < k < L-1$, we have:
      $$
      C(n_1,\ldots,n_L) \le C(n_1,\ldots,n_k) + C(n_{k+1}, \ldots, n_{L}).
      $$
    \item \label{pr: contractivity}
      (Contractivity)
      Capacity may only increase if a layer is duplicated. For example: 
      $$
      C(n,m,p) \le C(n, m, m, m, p).
      $$
    \item \label{pr: two layers}
      For any set $S \subset \R^n$: 
      $$
      C(S,n,m) = C(S)m.
      $$
    \item \label{pr: replace by image}
      For any set $S \subset \R^{n_1}$ and a threhsold map $f \in T(S,n_1,n_2)$, we have:
      $$
      C \big( f(S), n_2,n_3, \ldots, n_L \big) \le C(S, n_1,n_2, \ldots, n_L).
      $$
  \end{enumerate}
\end{lemma}

The proofs are elementary and left as an exercise.

\subsection{Known bounds on capacity of sets}
First, as a useful reminder, the following theorem about partitions of $\R^n$ by hyperplanes is well known (e.g. \cite{winder1966partitions}) and straightforward to prove by recurrence.
\begin{theorem}
\label{thm: hyperplane partition}
The number $K(m,n)$ of connected regions created by $m$ hyperplanes in $\R^n$ (passing through the origin) satisfies:
$$ K(m,n) \leq 2 \sum_{k=0}^{n-1} {m-1 \choose k} 
\label{eq:hyperplanes} $$
and the number $L(m,n)$ of connected regions created by $m$ affine hyperplanes in $\R^n$ satisfies:
$$ L(m,n) \leq \sum_{k=0}^{n} {m \choose k}.
\label{eq:affinehyperplanes} $$
In both cases, equality is achieved if the hyperplanes are in general position.
\end{theorem}

One of the most basic questions addressed in this paper revolves around the best upper and lower bounds on the capacity $C(S)$ in terms of  the
cardinality of $S$. The following upper bound is known:
%
%

\begin{lemma}[Capacity of sets: upper bound]	\label{lem: capacity elementary upper}
  The number of threshold functions on a given set $S \subset \R^n$ is bounded
  by
  $$
  2 \sum_{k=0}^n \binom{|S|-1}{k}.
  $$
  In particular, if $n \ge 4$ then: 
  $$
  C(S) \le 1 + n \log_2 \Big( \frac{e|S|}{n} \Big)
  \le n \log_2 |S|.
  $$
\end{lemma}

\begin{proof}
The first part of the lemma is presented in \cite[Section~4.2]{anthony2001discrete} and  
follows immediately from the first part of Theorem~\ref{thm: hyperplane partition}
by considering the number of regions into which $\R^n$ 
can be partitioned by the arrangement of $\vert S \vert$ hyperplanes of the form $x^\perp$, $x \in S$.
The second part of the Lemma can then be deduced from the first using 
the elementary bound on the binomial sums: 
$$
\sum_{k=0}^n \binom{N}{k} \le \Big( \frac{eN}{n} \Big)^n,
$$
which is valid for all integers $1 \le n \le N$, see e.g. \cite[Exercise~0.0.5]{vershynin2018high}.
The last part follows easily for $n \ge 4$.
\end{proof}

\begin{remark}[Tightness]				\label{rem: capacity on a set upper tightness}
  Lemma~\ref{lem: capacity elementary upper}
  gives the best possible upper bound on the capacity of a set $S \subset H^n$ 
  in terms of the cardinality of $S$. 
  In Section~\ref{s: enrichment}, we describe an enrichment method 
  that for a given $k \le n$ transforms the cube $H^k$ into a subset 
  $S \subset H^n$ of cardinality $|S|=2^k$ for which: 
  $$
  C(S) \asymp nk = n \log_2|S|.
  $$
   This shows that the bound in Lemma~\ref{lem: capacity elementary upper} is optimal 
   for almost any magnitude of the cardinality $|S|$. 
\end{remark}

\begin{lemma}[Capacity of sets: lower bound]		\label{lem: capacity elementary lower}
  For any finite set $S \subset \R^n$, there exists at least $2|S|$ threshold functions on $S$. 
  In particular:
  $$
  C(S) \ge \log_2|S| + 1.
  $$
\end{lemma}

\begin{proof}
The proof is elementary and we only sketch it.
The claim is easy to check for $n=1$. 
For general $n$, choose a projection $P$ in $\R^n$ onto some line 
and such that $P$ is injective on $S$. 
Then $C(S) \ge C(P(S))$. By affine invariance, we can realize $P(S)$ as a subset of $\R$
without changing the capacity. Then, applying the statement for $n=1$, we get $C(P(S)) \ge 2|S|$. 
Note that if $S=H^n$ the result can also be proved by noting that for any point of the hypercube there is a Boolean threshold function on $S$ that is equal to 1 on that point, and equal to 0 everywhere else. Including
all such functions and their negation yields the lower bound. 
\end{proof}

\begin{remark}[Tightness]		\label{rem: elementary lower tightness}
  The bound in Lemma~\ref{lem: capacity elementary lower} is generally tight:
  if the set $S$ lies on some line in $\R^n$, the there are exactly $2|S|$ threshold functions on $S$.
\end{remark}

Neverthelss, for many sets $S$ the lower bound given in Lemma~\ref{lem: capacity elementary lower} 
is too weak and can be improved. Consider, for example, the entire Boolean cube $S = H^n$. 
Lemmas~\ref{lem: capacity elementary upper} and \ref{lem: capacity elementary lower} give 
$\log n + 1 \le C(H^n) \le n^2$. As the following known result shows, 
the upper bound is tight, and the capacity of the Boolean cube is approximately $n^2$:

\begin{theorem}[Capacity of the Boolean cube]	\label{thm: Zuev}
  For any $n > 1$, we have:
  \begin{equation}	\label{eq: Zuev non-sharp}
    \frac{n(n-1)}{2} \le C(H^n) \le n^2.
  \end{equation}
  Moreover:
  \begin{equation}	\label{eq: Zuev sharp}
    C(H^n) = n^2 (1+o(1))
    \quad \text{as } n \to \infty.
  \end{equation}
\end{theorem}

The first, non-asymptotic, part of this theorem 
can be found in \cite[Theorems~4.3, 4.5]{anthony2001discrete}; see also \cite{cover1965geometrical, muroga1965lower}. It can also be derived from more general results in this paper: 
the upper bound on $C(H^n)$ follows from Lemmas~\ref{lem: capacity elementary upper} for $n \ge 4$,
and the lower bound on $C(H^n)$ is derived in Example~\ref{ex: Zuev} below.
The second, asymptotic, part of Theorem~\ref{thm: Zuev} 
was proved by Zuev \cite{zuev1991combinatorial}. 
{A tighter estimate corresponding to:
\begin{equation}
C(H^n)= n^2 - n \log_2 n \pm O(n) 
\label{eq:komlos}
\end{equation}
was obtained in \cite{kahn1995probability}.}

\begin{remark}[Extensions]
  Theorem~\ref{thm: Zuev} can be generalized to polynomial threshold functions \cite{baldi2018boolean} 
  of degree $d$, i.e. functions of the form $f(x) = h(p(x))$ where $p$ is a polynomial of degree $d$. 
  The capacity $C_d(H^n)$, defined as the binary logarithm of the number of such functions on $H^n$,
  satisfies:
  $$
  C_d(H^n)=\frac{n^{d+1} }{d!} (1+o(1))
  \quad \text{as } n \to \infty.
  $$
  thus generalizing Zuev's result \eqref{eq: Zuev sharp} which corresponds to $d=1$. 
  There exist further extensions of the capacity bounds for ReLU units, units with positive weights, 
  and units with binary weights; they are described in  \cite{baldi2018neuronal}.
\end{remark}

Armed with these definitions and preliminary results, we are set up to study the capacity of arbitrary feedforward architectures.

\subsection{Asymptotic notation}

In the estimation of various quantities, we will use the notation $\asymp$ and $\lesssim$ for identities and inequalities that hold up to constant factors. To be precise, 
$a \asymp b$ means that there exists two positive absolute constants $c_1$ and $c_2$ such that:
$$c_1 b \le a \le c_2 b.$$
Similarly, $a \lesssim b$ means that there exists a positive absolute constant $c$ such that:
$$a \le c b.$$ 
These notations are useful only when the quantities $a$ and $b$ vary as a function of certain parameters (e.g. layer sizes).
Positive absolute constants, which we denote by  $c_1, c_2, c, C, \ldots$ may not depend on anything, 
in particular on the number $L$ of layers or the number of nodes $n_k$ in any layer $k$.

\section{Overview of new results}				\label{s: new results}

\subsection{A capacity formula}

The main technical result of the paper is a two-sided bound on the capacity of 
fully-connected, layered, feedforward architectures $A(n_1,\ldots,n_L)$ 
with threshold transfer functions.

\begin{theorem}[Capacity formula] \label{thm: main}
  Consider a neural architecture $A(n_1,\ldots,n_L)$ with $L \ge 2$ layers. 
  Assume that the number of nodes in each layer satisfies
  $n_j > 18 \log_2 (L n_k)$ for any pair $j,k$ such that $1 \le j < k \le L$. Then:
  $$
  C(n_1,\ldots,n_L) 
  \asymp \sum_{k=1}^{L-1} \min(n_1,\ldots,n_k) n_k n_{k+1}.
  $$
\end{theorem}

The upper bound in Theorem~\ref{thm: main} is not difficult; 
we derive it in Section~\ref{s: upper bounds}
from Lemma~\ref{lem: capacity elementary upper} and the sub-additivity of the capacity.
The lower bound is significantly more challenging and requires new tools, which we call 
{\em multiplexing, enrichment, and stacking}. Once these tools are developed, we use them to prove the lower bound
in Section~\ref{s: lower bound}.

The upper bound in Theorem~\ref{thm: main} actually holds with the optimal factor $1$ if 
each non-output layer has at least four neurons (Proposition~\ref{prop: multi-channel upper}),
and it does not require the assumption $n_j \gtrsim \log_2 (L n_k)$. 
This mild assumption is important in the lower bound though to
prevents layer sizes from expanding too rapidly.
Although this assumption has an almost optimal form (Section~\ref{s: rapidly expanding}),
it can be slightly weakened (Section~\ref{s: smaller top layers}).

For the special single-neuron case $A(n,1)$, Theorem~\ref{thm: main} gives 
$$
C(n,1) \asymp n^2.
$$
Since $C(n,1) = C(H^n)$, this recovers the capacity estimate of the Boolean cube 
from Theorem~\ref{thm: Zuev} up to a constant factor. 
The proof of Theorem~\ref{thm: Zuev}, however, does not offer any insights on how to compute the capacity of deeper networks.

The simplest new case of Theorem~\ref{thm: main} is for networks $A(n,m,1)$ 
with one hidden layer, where it states that $C(n,m,1) \asymp n^2m + \min(n,m) m \asymp n^2 m$.
The constant factor implicit in this bound can be tightened for large $n$. 
Indeed, we will show in Corollary~\ref{cor: two layers} that:
\begin{equation}	\label{eq: two layers intro}
C(n,m,1) = n^2 m (1+o(1))
\end{equation}
if $n \to \infty$ and $\log_2m = o(n)$. This extends Zuev's  asymptotic result (Theorem~\ref{thm: Zuev}). 

\medskip

An immediate and somewhat surprising consequence of Theorem~\ref{thm: main} 
is that multiple output neurons can always be ``channeled'' through a single output neuron
without a significant change in capacity of the network:

\begin{corollary}	\label{cor: effect of output node}
  Under the assumptions of Theorem~\ref{thm: main}, we have:
  $$
  C(n_1,\ldots,n_{L-1},1) \asymp C(n_1,\ldots,n_{L-1}).
  $$
\end{corollary}

\begin{proof}
Comparing the capacity formulas for these two architectures, we see that all the terms in the two sums match except for the last (extra) term in $C(n_1,\ldots,n_{L-1},1)$, 
which is $\min(n_1,\ldots,n_{L-1}) n_{L-1}$. However, this term is clearly bounded
by $\min(n_1,\ldots,n_{L-2}) n_{L-2} n_{L-1}$, which is the last term in the capacity 
sum for $C(n_1,\ldots,n_{L-1})$. Therefore, the capacity sums for the two architectures
are within a factor of $2$ from each other. 
\end{proof}

Let us mention that the capacity formula in Theorem~\ref{thm: main} obtained for inputs in $H^{n_1}$ can be extended to inputs from other finite sets $C(S,n_1,n_2,\ldots,n_L)$. 
In Propositions~\ref{prop: multi-channel upper} and \ref{prop: restricted vs unrestricted} 
we give upper and lower bounds on this variation of the 
capacity in terms of the cardinality of $S$.

\subsection{Networks achieving maximal capacity}		\label{s: max capacity intro}

We can use the capacity formula in Theorem~\ref{thm: main} to find 
networks that maximize the capacity subject to natural constraints.
Here we find the most capable networks (a) with a given number of connections (weights) and 
(b) with a given number of nodes (neurons). 

Let us start with (a). The number of connections, or synaptic weights, of the neural architecture 
$A(n_1,\ldots,n_L)$ is 
$$
W = W(n_1,\ldots,n_L) = \sum_{k=1}^{L-1} n_k n_{k+1}. 
$$
Fixing $W$ makes sense because it is approximately the same as fixing the number 
of {\em parameters} $P$ of the neural architecture. The difference between $P$ and $W$ are the biases of the neurons so that: 
$P = W + n_2 + \cdots + n_L$. Thus, we always have $W \le P \le 2W$. Furthermore, since 
the number of neurons is usually much smaller than the number of connections $W$, 
in most situations $P$ approximately equals $W$. We have the following Corollary.

\begin{corollary}[Optimal network with given number of connections]	\label{cor: max capacity given parameters}
  Under the conditions of Theorem~\ref{thm: main}, we have:
  $$
  C(n_1,\ldots,n_L) \le n_1 W. 
  $$
  Moreover, any network satisfying $n_1 \leq n_k$ for $k=2,\ldots, L-1$
 approximately achieves maximal capacity:
   $$
  C(n_1,\ldots,n_L) \asymp n_1 W. 
  $$
\end{corollary}

\begin{proof}
The first statement is known and follows from prior results on the growth 
function of general (not necessarily fully connected) networks
\cite[Corollary~3]{baum1989size}.
It also trivially follows from Theorem~\ref{thm: main} and the fact that 
$\min(n_1,\ldots,n_k) \le n_1$. 
The second statement follows from 
 Theorem~\ref{thm: main}
and the fact that $\min(n_1,\ldots,n_k) = n_1$
under the assumptions of the Corollary.
\end{proof}

Examples of standard architectures that satisfy the condition of the second statement of the Corollary
include monotonically expansive feedforward networks
satisfying $n_1 \leq n_2\leq \cdots  \leq n_{L-1}$) (the output layer can be expansive or contractive)
satisfy the conditions of the Corollary. Likewise, 
expansive autoencoders networks satisfying $n_1 \leq n_2$ and $n_3=n_1$ (in the case of a single hidden layer) also satisfy the condition of the Corollary.
Finally any shallow network with a single hidden layer, where the hidden layer is larger than the input ($n_2 \geq n_1$), satisfies the condition of the Corollary and thus approximately achieves maximal capacity.
In contrast, in many deep forward networks used in applications there exists layers that are smaller in size than the input layer and thus these networks 
do not achieve maximal capacity. On the positive side, this implies that such networks 
do not require $W$ independent examples for their training.

Next, let us find the most capable network with a given number of nodes, or neurons. The constraint on the number of neurons is loosely inspired by biological situations where the number of neurons may stay approximately constant, but the number and pattern of connections among the neurons may vary.

It turns out that for a fixed number $N$ of nodes
$$
N = n_1 + \cdots + n_L
$$
{\em the most capable networks are shallow}.
To quickly see why, note that Theorem~\ref{thm: main} yields:
$$
C \Big( \frac{N}{L}, \ldots, \frac{N}{L} \Big) \asymp \frac{N^3}{L^2}.
$$
This shows that the capacity decreases if we rearrange the fixed set of nodes into more layers. 
Furthermore, we can identify the most capable neural architectures with given number of nodes:

\begin{corollary}[Optimal network with given number of neurons: informal statement]	\label{cor: optimization intro}
  Among all neural architectures $A(n_1,\ldots,n_L)$ with a given number of nodes
  $N = n_1 + \cdots + n_L$, the architecture $A(2N/3, N/3)$ approximately maximizes capacity.
  
  Suppose that in addition to fixing $N$, we also fix the number of input neurons $n_1$.
  Then:
  \begin{enumerate}[\quad 1.]
    \item If $n_1 < N/2$, the architecture $A(n_1,N/2,N/2-n_1)$ approximately maximizes capacity.
    \item If $n_1 \ge N/2$, the architecture $A(n_1,N-n_1)$ approximately maximizes capacity.
  \end{enumerate}
\end{corollary}

This result is formally stated in Theorems~\ref{thm: max capacity given nodes} 
and \ref{thm: max capacity given nodes and input}. 
Due to the equivalence \eqref{cor: effect of output node}, similar results hold for architectures 
$A(n_1,\ldots,n_L,1)$ with a single output unit, as well for architectures with a fixed number of output units. 
Complementary minimization results (Theorem~\ref{thm: min capacity given weights/neurons and input}) show that, under fixed budgets of units or connections, the capacity is minimized by the deepest possible networks, those with a single unit in each hidden layer. 

\medskip

These optimization results go against the belief, held by some, that deep architectures are more powerful because they can compute more functions than shallow architectures. The contrary is actually true: everything else being equal, deep architectures tend to compute less functions, but the functions they compute are more `` interesting'',  or have ``better properties''. This is related to the well-known regularizing effect of deep learning: deep architectures tend to avoid overfitting, even when the amount of training data is small compared to the number of parameters. While some of this regularizing effect can be attributed to learning methods based on stochastic gradient descent, our analysis shows that there is a strong structural component (Section~\ref{s: structural regularization}).

\subsection{Capacity of sets}

The derivation of the capacity formula (Theorem~\ref{thm: main}) is based on 
new bounds on the capacity of finite sets. 
In Lemmas~\ref{lem: capacity elementary lower} and \ref{lem: capacity elementary upper}
we noted the upper and lower bounds 
\begin{equation}	\label{eq: capacity vs cardinality intro}
\log_2|S| + 1 \le C(S) \le n \log_2|S|,
\end{equation}
which hold for any finite set $S$ in $\R^n$. 
We mentioned that both bounds are generally best possible. 
Surprisingly, the lower bound can be significantly improved for subsets of the Boolean cube $H^n$. 
Indeed, the main result of Section~\ref{s: capacity set lower} states the following:

\begin{theorem}[Capacity of a set]		\label{thm: BTF on general S intro}
  The capacity of any set $S \subset H^n$ satisfies:
  $$
  C(S) > \frac{1}{16} \log_2^2 |S|.
  $$ 
\end{theorem}

This new bound is tight up to an absolute constant factor. 
Indeed, if $S = H^k$ is a Boolean cube canonically embedded in $H^n$, 
Theorem~\ref{thm: BTF on general S intro} gives $C(S) \gtrsim k^2$, 
which matches the upper bound $C(S) = C(H^k) \le k^2$ in Theorem~\ref{thm: Zuev}. 

Unfortunately, even the new lower bound may be too weak for some applications.  
In particular, we need a stronger result to prove Theorem~\ref{thm: main} even for three layers ($L=3$). 
Thus one may wonder if in some sense {\em the capacity of $S$ could be increased through some preprocessing of $S$}. Specifically, can we transform $S$ into a set $F(S)$ whose capacity is significantly larger, 
ideally as large as the upper bound in \eqref{eq: capacity vs cardinality} allows? Furthermore,
in doing so, we would like to stay in the category of subsets of the Boolean cube and 
use only transformations $F$ that a network of threshold units can compute. Specifically, we will require
that the {\em enrichment map} $F$ be a threshold map $F \in T(H^n,H^m)$.
We address the enrichment problem in the particular case where $S = H^n$, leaving the general 
case for future investigations. The main result of Section~\ref{s: enrichment} states the following.

\begin{theorem}[Enrichment]		\label{thm: enrichment intro}
  Let $n$ and $m$ be positive integers satisfying $n \le m \le 2^{n/2}$.
  There exists an injective threshold map $F \in T(H^n,H^m)$ such that:
  $$
  C \big( F(H^n) \big) \asymp nm.
  $$
\end{theorem}

The enrichment map $F$ transforms the cube $S = H^n$ into the set
$S' := F(S) \subset H^m$. The enriched set $S'$ has the same cardinality as $S$
and almost the maximally possible capacity:
$$
C(S') \asymp nm = m \log_2 |S'|,
$$
which matches the upper bound in Lemma~\ref{lem: capacity elementary upper} in dimension $m$.

\subsection{Capacity of networks: new tools}

In addition to the new bounds on capacity of sets, 
our proof of Theorem~\ref{thm: main} uses some other new tools,
which may be helpful in other applications. 
Let us briefly explain our argument.

The upper bound in Theorem~\ref{thm: main} can be quickly derived 
from the upper bound in Lemma~\ref{lem: capacity elementary upper} and the sub-additivity of capacity. 
A similar argument was used before to obtain upper 
bounds on the VC-dimension of neural networks, see e.g. \cite{bartlett1999almost}.

The matching lower bound is considerably harder to prove. 
For networks with one hidden layer, the proof of \eqref{eq: two layers intro} 
is based on a new method that is inspired by 
the idea of {\em multiplexing} in
signal processing (Section~\ref{s: multiplexing}).
Recall that estimating the capacity $C(n,m,1)$ of a two-layer network 
involves counting all functions $\phi \circ f$, where
$$
f = (f_1,\ldots,f_m): H^n \to H^m
$$
is a threshold map (i.e. a map whose all components $f_i$ are threshold functions)
and
$$
\phi: H^m \to \{0,1\}
$$ 
is a threshold function.
Due to Theorem~\ref{thm: Zuev}, there are approximately $(2^{n^2})^m = 2^{n^2 m}$ 
different functions $f$. However, this does not yield any lower bound on
the number of compositions $\phi \circ f$: it might happen that two different functions
$f$, when composed with $\phi$, produce the same function. The multiplexing method circumvents this issue by combining two signals: a selector signal, and a threshold map signal. 
It allows the network to compute any one of the $m$ components $f_i$ of $f$;
the first $\log_2 m$ bits of the input vector $x \in H^n$
act as {\em selector} bits used to select which component $f_i$ of the map should be in the output. 

Next, the capacity of networks with two hidden layers $C(n,m,p,1)$ is handled by combining 
multiplexing with {\em enrichment} (Section~\ref{s: enrichment}). 
A fixed enrichment map $F: H^n \to H^m$
whose existence is guaranteed
by Theorem~\ref{thm: enrichment intro} is used to connect the first two layers of the network. 
The fact that the image of $F$ has large capacity gives us plenty of different threshold maps 
$G: H^m \to H^p$ between the two hidden layers. Multiplexing is then used 
to preserve the multitude of functions in $G$ when they are composed with an output 
function $\phi: H^p \to \{0,1\}$.

Finally, to handle networks with arbitrarily many layers (Section~\ref{s: stacking}), 
we {\em stack} three-layer networks
in a particular way to ensure that: (1) they may perform computations independently; and (2) 
the number of nodes in each layer is at most $n_k$. Figure~\ref{fig: stacking} 
illustrates the stacking method. Then Theorem~\ref{thm: main} can be deduced from the 
capacity analysis of three-layer networks and their stacking.

\subsection{Paper roadmap}

In Section~\ref{s: examples}, we give a few basic examples of threshold functions and maps. 
In Section~\ref{s: upper bounds}, we derive upper bounds on the capacity of networks, 
and in particular the upper bound in Theorem~\ref{thm: main}.
The reader interested only in the proof of Theorem~\ref{thm: main} may then skip to
Section~\ref{s: multiplexing}. In Sections~\ref{s: product sets}, we develop combinatorial tools
for the analysis of the capacity of sets. We use these tools in Section~\ref{s: capacity set lower} 
to prove the main result on the capacity of subsets of the Boolean cube, 
Theorem~\ref{thm: BTF on general S intro}.
In Section~\ref{s: multiplexing}, we develop the multiplexing technique and use it to estimate the
capacity of networks with one hidden layer. 
In Section~\ref{s: enrichment}, we prove the Enrichment Theorem~\ref{thm: enrichment intro} 
and use it to handle networks with two hidden layers. 
In Section~\ref{s: stacking}, we extend the resulst to arbitrary many layers,
by stacking three-layer networks, and complete the proof of Theorem~\ref{thm: main}.
In Section~\ref{s: max capacity}, we study networks with maximal or minimal capacity, and in particular prove a rigorous version of Corollary~\ref{cor: optimization intro}. Section~\ref{s: structural regularization}, addresses the issue of structural regularization. Finally several open questions are discussed in the conclusion
(Section~\ref{s: open}).

\section{Useful examples of threshold maps}		\label{s: examples}

In this section we give several examples of threshold functions and threshold maps. 
These examples will become useful in the proofs of the main results.

Throughout this paper, the symbol $\oplus$ denotes the direct sum. 
For two vectors $a \in \R^n$ and $b \in \R^m$, the direct sum $a \oplus b \in \R^{n+m}$
is obtained by concatenation of $a$ and $b$. 
For two sets $A \subset \R^n$ and $B \subset \R^m$, the direct sum 
$A \oplus B \subset \R^{n+m}$ is defined as:
$$
A \oplus B = \{ a \oplus b :\; a \in A, \; b \in B \}.
$$
A similar notation is used for the direct sum of a set and a vector, for example:
$$
A \oplus b = A \oplus \{b\} = \{ a \oplus b :\; a \in A \}.
$$

\subsection{Examples of threshold functions}

It is well known and trivial to prove that the Boolean negation \textsc{not} is a threshold function 
on $H^1$, and the Boolean functions of $n$ variables 
\textsc{and} ($x_1 \wedge \cdots \wedge x_n$), \textsc{or}
($x_1 \vee \cdots \vee x_n$), and their negations \textsc{nand} and \textsc{nor}, 
are all threshold functions on $H^n$.
Note that the \textsc{and} operation $x_1 \wedge \cdots \wedge x_n$ amounts to checking whether all $x_i$ are equal to $1$. The value $1$ is not special and can be replaced by any real number $\theta_i$:

\begin{lemma}		\label{lem: checking equality}
  Consider the function on $H^n$ that checks whether the argument 
  equals a given vector $\theta \in \R^n$:
  $$
  f(x) = (x = \theta)= 
  \begin{cases}
    1 & \text{if } x = \theta \\
    0 & \text{if } x \ne \theta
  \end{cases}.
  $$
Then $f$ is a Boolean threshold function, i.e. $f \in T(H^n)$.
\end{lemma}

\begin{proof}
We can assume without any loss of generality that $\theta = (\theta_1,\ldots,\theta_n) \in H^n$,
for otherwise $f$ is the zero function and trivially lies in $T(H^n)$. 
Let $m = \sum_{i=1}^n \theta_i$.
Now, $f$ can be expressed as: 
\begin{equation}	\label{eq: f as TF}
f(x) = h \Big( 2\ip{\theta}{x} - \sum_{i=1}^n x_i -m + \frac{1}{2} \Big)
\end{equation}
and therefore $f$ is a threshold function. 
Indeed, if $x = \theta$ then $\ip{\theta}{x} = m$ and the right hand side of \eqref{eq: f as TF}
is equal to $h(1/2) = 1$. 
If $x \ne \theta$, we consider two cases: $\sum_i x_i > m$ and 
$\sum_i x_i \leq m$. It is easy to check that in each one of these cases, 
the argument of $h$ in \eqref{eq: f as TF} is $-1/2$ or less, and thus $f=0$.
\end{proof}

Lemma~\ref{lem: checking equality} can be generalized one step further. It is possible to combine two operations
into one threshold function: check whether the argument equals $\theta$, and compute a given Boolean threshold function $f$.

\begin{lemma}[Adding a clause]		\label{lem: adding a clause}
  Consider a Boolean threshold function $f \in T(H^n)$ and a vector $\theta \in \R^q$.
  Then the function
  $$
  g(x \oplus y) \coloneqq f(x) \wedge (y=\theta) = 
  \begin{cases}
    f(x) & \text{if } y=\theta \\
    0 & \text{if } y \ne \theta
  \end{cases}
  $$
  is a Boolean threshold function, i.e. $g \in T(H^{n+q})$.
\end{lemma}

\begin{proof}
We can assume without loss of generality that $\theta \in H^q$,
for otherwise $g$ is the zero function and trivially lies in $T(H^{n+q})$. 
Let $m = \sum_{i=1}^q \theta_i$.
Express $f \in T(H^n)$ as: 
$$
f(x) = h \big( \ip{a}{x} + \a \big)
$$
for suitable $a \in \R^n$ and $\a \in \R$. 
Choose any suitable constants $K>0$ and $b>0$ such that
$K(\ip{a}{x} + \a-b)$ is in the $[-1/2,0]$ interval for all the 
$x$ that satisfy $f(x)=1$, and in the $(-\infty, -1/2)$ interval for all the 
$x$ that satisfy $f(x)=0$. We claim that $g$ can be expressed as 
$$
g(x\oplus y ) = h \left( K \big( \ip{a}{x} + \a-b \big) +   2\ip{\theta}{y} - 
\sum_{i=1}^q y_i - q + \frac{1}{2} \right)
$$
and therefore $g$ is a threshold function.
Indeed, we have seen in the proof of Lemma~\ref{lem: checking equality} that the quantity
$2\ip{\theta}{y} - \sum_{i=1}^q y_i - q + 1/2$ is either equal to $1/2$ when $y=\theta$, or at most 
$-1/2$ for all other values of $y$. It is then easy to check that when $y=\theta$, we have 
$g(x \oplus y)=f(x)$, and when $y \not = \theta$, we have $g(x \oplus y)=0$.
\end{proof}

Lemma~\ref{lem: adding a clause} easily generalize to 
functions computable by fedforward neural networks.

\begin{lemma}[Adding a clause]			\label{lem: adding a clause network}
  Consider a function $f \in T(n_1,\ldots,n_L,1)$ and a vector $\theta \in \R^q$. 
  Then the function:
  $$
  g(x \oplus y) \coloneqq f(x) \wedge (y=\theta) = 
  \begin{cases}
    f(x) & \text{if } y=\theta \\
    0 & \text{if } y \ne \theta
  \end{cases}
  $$
  satisfies $g \in T(n_1+q,\ldots,n_L+q,1)$.
\end{lemma}

The proof is elementary:
it suffices to use the identity map on the additional $q$ coordinates up to the top layer.


\subsection{Examples of threshold maps}	

Let us go over some basic examples of threshold maps.
Obviously, these include the identity map on $H^n$ and all threshold functions.
The next lemma gives a more interesting example.

\begin{lemma}[Exponential map]				\label{lem: sparsifier}
  Fix an integer $k$ and let $\{e_i\}$ denote the canonical vector basis in $\R^{2^k}$. 
  Then any one-to-one map: 
  $$
  f :\; H^k \to \{e_1,\ldots,e_{2^k}\}
  $$
  is a Boolean threshold map, i.e. $f \in T(H^k,k,2^k)$.
\end{lemma}

\begin{proof}
The components of the map $f = (f_1,\ldots,f_{2^k})$ trivially satisfy the following: 
$f_i(x)$ equals $1$ if $f(x) = e_i$ and $0$ otherwise. 
The last condition can be written as $x = f^{-1}(e_i)$. 
Then Lemma~\ref{lem: checking equality} implies that $f_i$ is a threshold function, 
and hence $f$ is a threshold map. 
\end{proof}

A specific example of $f$ in Lemma~\ref{lem: sparsifier} is the {\em exponential map}, 
which interprets the input vector $x \in H^k$ as a binary representation of a number 
and returns the binary representation of $2^x$. 
For example, if $k=2$, then: 
$$
f(00) = 2^0 = (0000), \; 
f(01) = 2^1 = (0010), \; 
f(10) = 2^2 = (0100), \; 
f(11) = 2^3 = (1000). 
$$

\section{Capacity of networks: upper bounds}				\label{s: upper bounds}

In this section, we prove general upper bounds on the capacity of neural networks,
from which the upper bound in Theorem~\ref{thm: main} will follow as a special case. The results rely on the following key remark.

\begin{remark}		\label{rem: bottlenecks}
The capacity of a network is always upper bounded by the sum of the capacities of its neurons. However, in general this is a weak bound due to the restrictions in capacity posed by bottle-neck layers. To see this consider two consecutive layers $k$ and $k+1$. In principle, a unit in layer $k+1$ could have capacity 
of the order of $n_k^2$ by Theorem~\ref{thm: Zuev}. However, if there is a layer $i<k$ with $n_i < n_k$, for any setting of the weights, the units in layer $k$ can only take at most $2^{n_i}$ values, rather than $2^{n_k}$. 
By Lemma~\ref{lem: capacity elementary upper}, this will reduce the capacity of a unit in layer $k+1$ to be at most of the order of $n_k n_i$ instead of $n_k^2$. The same effect is seen if the values of the input layer are restricted. 
\end{remark}

\begin{proposition}[Capacity formula: upper bounds]		\label{prop: multi-channel upper}
  For any $L \ge 2$ and $n_1, \ldots, n_{L-1} \ge 4$, $n_L \ge 1$, the following holds. 
  Consider a finite set $S \subset \R^{n_1}$ and let $n = \log_2|S|$. 
  Then:  
  $$
  C(S,n_1,n_2,\ldots,n_L) 
  \le n n_1 n_2 + \sum_{k=2}^{L-1} \min(n,n_2,\ldots,n_k) n_k n_{k+1}.
  $$
  In particular, we have:
  \begin{equation}	\label{eq: multi-channel upper}
  C(n_1,n_2,\ldots,n_L) 
  \le \sum_{k=1}^{L-1} \min(n_1,\ldots,n_k) n_k n_{k+1}.
  \end{equation}
\end{proposition}

\begin{proof}
The proof is by induction. First consider the case where $L=2$. 
Using property~\ref{pr: two layers} in Lemma~\ref{lem: capacity basic},
the capacity bound from Lemma~\ref{lem: capacity elementary upper}, and the assumptions on $S$, 
we see that:
\begin{equation}	\label{eq: CSn2}
C(S,n_1,n_2) = C(S) n_2 \le (n_1 \log_2|S|) n_2 \le n_1 n n_2,
\end{equation}
which is the claimed bound.
Assume the property is true for $L$ layers. To prove it for 
$L+1$, just apply Remark~\ref{rem: bottlenecks} noting that the top layer contains $n_{L+1}$ units, and the capacity of each unit is at most $n_L \cdot \min(n,n_2,\ldots,n_L)$.
This completes the proof of the first inequality. The second inequality is obtained simply by letting
$S=H^{n_1}$ (i.e. $n=n_1$).
\end{proof}

The assumption that each non-output layer should have at least four neurons
can be removed from Proposition~\ref{prop: multi-channel upper} at the cost of
an absolute constant factor in the capacity formula.

\begin{corollary}[Upper bound in Theorem~\ref{thm: main}] 	\label{cor: main upper bound}
  For any $L \ge 2$ and any $n_1,\ldots,n_L \ge 1$, we have
  $$
  C(n_1,n_2,\ldots,n_L) 
  \lesssim \sum_{k=1}^{L-1} \min(n_1,\ldots,n_k) n_k n_{k+1}.
  $$
\end{corollary}

\begin{proof}
Apply Proposition~\ref{prop: multi-channel upper} for the capacity
$C(4n_1,\ldots,4n_{L-1}, n_L)$ and note that $C(n_1,n_2,\ldots,n_L)$ can only be smaller.
\end{proof}

In summary, we have derived the general upper bound 
associated with Theorem~\ref{thm: main}, 
and shown that the upper bound holds with an absolute and optimal 
constant factor of $1$ in the general case, where each non-output layer has at least four nodes.

Finally, let us note that the same capacity bound holds if we extend the network by adding a single output node.

\begin{corollary}[Adding an output node] 	\label{cor: main upper bound 1}
  For any $L \ge 2$ and any $n_1,\ldots,n_L \ge 1$, we have
  $$
  C(n_1,n_2,\ldots,n_L,1) 
  \lesssim \sum_{k=1}^{L-1} \min(n_1,\ldots,n_k) n_k n_{k+1}.
  $$
\end{corollary}

The argument to prove this result is the same as the argument used to prove Corollary~\ref{cor: effect of output node}.

\section{Capacity of product sets: slicing}		\label{s: product sets}

Now that we have good upper bounds on the capacity of sets and 
neural networks, we turn to the lower bounds. 
In Lemma~\ref{lem: capacity elementary lower} we noted the elementary lower bound:
\begin{equation}	\label{eq: capacity lower trivial again}
C(S) \ge \log_2|S|+1,
\end{equation}
which is valid for any set $S \subset \R^n$.
We observed in Remark~\ref{rem: elementary lower tightness} that this bound is in general tight.
Nevertheless, it can often be improved if additional information about the set $S$ is available. 
In Section~\ref{s: capacity set lower}, we will show that 
if $S$ is a subset of the Boolean cube, then the lower bound in \eqref{eq: capacity lower trivial again} can be significantly improved . 
In this section, we develop general combinatorial tools that will be needed to derive the improved lower bound. 

Early lower bounds on $C(H^n)$, 
and in particular the lower bound in \eqref{eq: Zuev non-sharp},
were based on simple combinatorial considerations and induction \cite{muroga1965lower,anthony2001discrete}.
In this section, we extend these combinatorial methods in order to be able to handle capacities $C(S)$ of arbitrary sets $S$. 
Although the methods can be applied to any subset $S \subset H^n$, the best results are obtained when $S$ has a product structure, as explained below.

\subsection{Slicing}
The following theorem relates the capacity of a general set $S$ to the capacities and cardinalities of 
the {\em slices} of $S$. A slice is obtained by fixing the values of certain coordinates. 
For example, the elements of $S$ whose first four coordinates are $1011$ form a slice of $S$. 
By monotonicity, the capacity of $S$ is lower bounded by the capacity of any slice of $S$. 
This trivial bound can be boosted if, in addition, other slices have many points. Let us show this.

\begin{theorem}[Slicing]		\label{thm: slicing}
  Let $u_1,\ldots,u_k \in \R^m$ be a linearly independent set of vectors,
  and let $V_1,\ldots,V_k \subset \R^n$ be arbitrary finite sets. 
  Consider the subset $S \subset \R^{m+n}$ whose fibers at $u_i$ are $V_i$, i.e. let:
  $$
  S := \bigcup_{i=1}^k u_i \oplus V_i.
  $$ 
  Then the number of threshold functions on $S$
  satisfies:
  $$
  |T(S)| \ge |T(V_1)| \cdot \big( |V_2|+1 \big) \cdot \big( |V_3|+1 \big) \cdots \big( |V_k|+1 \big).
  $$
\end{theorem}

The proof of Theorem~\ref{thm: slicing} is based on a lifting trick.
Given a vector $a \in \R^n$ and a set $V \subset \R^n$, 
let us denote by $T_a(V)$ the set of all functions $f \in T(V)$ 
that can be expressed as: 
\begin{equation}	\label{eq: f a alpha}
f(x) = f_{a,\alpha}(x) = h \big( \ip{a}{x} + \alpha \big) 
\end{equation}
for some $\a \in \R$.
Thus, the functions in $T_a(V)$ are obtained by ``cloning'', i.e. by fixing $a \in \R^n$ 
and varying a single parameter -- the bias $\alpha \in \R$. 

\begin{lemma}
  If a vector $a \in \R^n$ separates the points of a finite set $V \subset \R^n$, then:  
  \begin{equation}	\label{eq: clones}
  |T_a(V)| = |V|+1,
  \end{equation}
  i.e. every $f$ has exactly $|V|+1$ different clones.
\end{lemma}

\begin{proof}
Let $V = \{x_1,\ldots,x_N\}$, where the points $x_i$ are ordered so that 
the sequence $t_i := -\ip{a}{x_i}$ is increasing with $i$. 
Now increase $\a$ continuously from $-\infty$ to $\infty$. 
As $\a$ crosses a point $t_i$, the function $f_{a,\alpha}(x) = h \big( \ip{a}{x} + \alpha \big)$
changes (since it changes its value on $x_i$ from $0$ to $1$), and there are no other points $\a$
where $f_{a,\alpha}(x)$ changes.
The $N$ crossover points $t_i$ partition $\R$ into $N+1$ intervals. 
Each interval corresponds to a different function $f_{a,\alpha}(x)$. 
Thus, the total number of such functions is $N+1$.
\end{proof}
 
The lifting trick described in the next lemma allows us to combine $k$ given clones of $f$ 
into a single threshold function on a larger domain.

\begin{lemma}[Lifting]			\label{lem: lifting}
  Let $u_1,\ldots,u_k \in \R^m$ be a linearly independent set of vectors,
  and let $V_1,\ldots,V_k \subset \R^n$ be arbitrary finite sets. 
  Fix $a \in \R^n$ and consider any functions 
  $f_i \in T_a(V_i)$, $i=1,\ldots,k$.
  Then we can find a function $F = F_{f_1,\ldots,f_k} \in T(S)$ such that: 
  $$
  F(u_i \oplus \cdot) = f_i, \quad i=1,\ldots,k.
  $$
\end{lemma}

\begin{proof}
By definition, the functions $f_i$ can be expressed as:
$$  
f_i(x) = h \big( \ip{a}{x} + \a_i \big). 
$$
Since the vectors $u_i \in \R^m$ are linearly independent, there exist $b \in \R^m$ such that:
$$
\ip{b}{u_i} = \a_i, \quad i=1,\ldots,k.
$$
For any vectors $u \in \R^m$ and $x \in \R^n$, define:
$$
F(u \oplus x) \coloneqq h \big( \ip{b}{u} + \ip{a}{x} \big).
$$
Obviously, $F$ is a threshold function on $\R^{m+n}$,
and by restricting it to $S$ we can say that $F \in T(S)$.
Now:
$$
F(u_i,x) 
= h \big( \ip{b}{u_i} + \ip{a}{x} \big)
= h \big( \a_i + \ip{a}{x} \big)
= f_i(x).
$$
The lemma is proved.
\end{proof}

\medskip

\begin{proof}[Proof of Theorem~\ref{thm: slicing}.]
For each $f \in T(V_1)$, let us choose and fix a vector $a = a(f) \in \R^n$ 
so that \eqref{eq: f a alpha} holds. Moreover, we can always choose $a$ so that it
separates the points of $V_2 \cup \cdots \cup V_k$, i.e. so that 
$\ip{a}{x} \ne \ip{a}{y}$ for any distinct pair of points $x,y \in V_2 \cup \cdots \cup V_k$.
(This can be done by perturbing $a$ slightly. Such perturbation does not change the function $f$ 
but allows $a$ to separate points.)

Consider the set of all $k$-tuples of functions 
$(f,f_2,f_3,\ldots,f_k)$ where $f_i \in T_a(V_i)$ for each $i$. 
Each such tuple consists of a function $f \in T(V_1)$ and some $k-1$ ``clones'' $f_i$ of $f$.
Due to \eqref{eq: clones}, each clone $f_i$ in the tuple can be chosen in exactly $|V_i|+1$ ways.
Thus the number of all such tuples is: 
\begin{equation}	\label{eq: tuple count}
|T(V_1)| \cdot \big( |V_2|+1 \big) \cdot \big( |V_3|+1 \big) \cdots \big( |V_k|+1 \big).
\end{equation}

Lemma~\ref{lem: lifting} implies that different tuples $(f,f_2,f_3,\ldots,f_k)$
produce different liftings $F = F_{f,f_2,f_3,\ldots,f_k} \in T(S)$. 
Indeed, one can uniquely recover the tuple from the fibers $F(u_i \oplus \cdot)$ of $F$.

Summarizing, $T(S)$ is lower bounded by the number of different liftings $F$, 
which in turn is lower bounded by the number of different tuples $(f,f_2,f_3,\ldots,f_k)$,
which finally is lower bounded by the expression in \eqref{eq: tuple count}, completing the proof of Theorem~\ref{thm: slicing}.
\end{proof}

\subsection{Product sets}

Theorem~\ref{thm: slicing} is especially useful when $S$ is a product of sets. 

\begin{corollary}[Capacity of product sets]		\label{cor: CSp2}
  Let $U \subset \R^m$ and $V \subset \R^n$ be finite sets.
  If $U$ is linearly independent, then: 
  $$
  C(U \oplus V) \ge (|U|-1) \log_2|V| + C(V).
  $$  
\end{corollary}

\begin{proof}
If $U = \{u_1,\ldots,u_k\}$, we can write $U \oplus V = \bigcup_{i=1}^k u_i \oplus V$.
Applying Theorem~\ref{thm: slicing} for $V_i = V$, we get:
$$
|T(U \oplus V)| \ge |T(V)| \cdot |V|^{|U|-1}.
$$
Taking logarithms of both sides completes the proof.
\end{proof}

By induction, this bound extends to products of arbitrary many sets.

\begin{corollary}[Capacity of product sets]		\label{cor: CSp}
  Assume $S_p = U \oplus \cdots \oplus U \in \R^{pm}$ is the product 
  of $p >1$ copies of a linearly independent 
  subset $U \subset \R^m$ with $|U| > 1$. Then:
  $$
  C(S_p) \ge \frac{1}{8} p^2 |U| \log_2|U|.
  $$
\end{corollary}

\begin{proof}
Apply Corollary~\ref{cor: CSp2} for the sets 
$U$ and $V = S_{p-1}$, whose cardinalities are $k \coloneqq |U|$ and $|V| = k^{p-1}$, and get:
$$
C(S_p) = C(U \oplus S_{p-1}) 
\ge (p-1) (k-1) \log_2 k + C(S_{p-1}).
$$
Apply Corollary~\ref{cor: CSp2} again for $S_{p-1} = U \oplus S_{p-2}$. 
Continuing in this way $p-1$ times, we obtain:
\begin{equation}	\label{eq: CSp prelim}
C(S_p) \ge \Big( (p-1) + (p-2) + \cdots + 1 \Big) (k-1) \log_2 k + C(U).
\end{equation}
Now:
$$
(p-1) + (p-2) + \cdots + 1 = \frac{p(p-1)}{2} \ge \frac{p^2}{4}
$$
as $p \ge 2$, $k-1 \ge k/2$ as $k \ge 2$, and $C(U) \ge 0$.
Substituting this into \eqref{eq: CSp prelim} completes the proof. 
\end{proof}

\begin{remark}[Relaxing the linear independence assumption]	\label{rem: relaxing linear independence}
In the main results of this section, we assumed that the set 
$U = \{u_1,\ldots,u_k\} \subset \R^m$ is linear independent.
This could be relaxed by assuming only that the set: 
$$
U \oplus 1 = \{u_1 \oplus 1, \ldots, u_k \oplus 1\} \in \R^{m+1}
$$ 
be linearly independent.

To see this, modify the argument in the Lifting Lemma~\ref{lem: lifting} as follows. 
Since the vectors $u_i \oplus 1 \in \R^{m+1}$ are linearly independent, 
there exists a vector $b \oplus \beta \in \R^{m+1}$ such that
$\ip{b \oplus \beta}{u_i \oplus 1} = \ip{b}{u_i} + \beta = \a_i$ for all $i=1,\ldots,k$.
Now define $F(u \oplus x) \coloneqq h \left( \ip{b}{u} + \beta + \ip{a}{x} \right)$.
\end{remark}

\subsection{Totally separated sets}			\label{s: totally separated}

In addition to product sets, Theorem~\ref{thm: slicing} can easily be specialized to
totally separated sets. 

\begin{definition}			\label{def: totally separated}
  Two subsets $A$ and $B$ of $\R^n$
  are {\em totally separated} if they lie in two different parallel hyperplanes of $\R^n$. 
\end{definition}

\begin{lemma}		\label{lem: totally separated}
  Let $S \subset \R^n$. If $A$ and $B$ are totally separated subsets of $S$ then: 
  $$
  |T(S)| \ge |T(A)| \cdot \big( |B|+1 \big).
  $$
\end{lemma}

\begin{proof}
By affine invariance, we may assume that 
$e_1 = (1,0,\ldots,0)$ is a normal vector to both hyperplanes in which $A$ and $B$ lie. 
Thus, we can express:
$$
A = u_1 \oplus V_1, \quad B = u_2 \oplus V_2
$$
for some distinct numbers $u_1, u_2 \in \R$ and sets $V_1, V_2 \subset \R^{n-1}$.
Moreover, since $u_1 \ne u_2$, the vectors $u_1 \oplus 1 = (u_1,1)$ and $u_2 \oplus 1 = (u_2,1)$ 
are linearly independent in $\R^2$.
Applying Theorem~\ref{thm: slicing} together with Remark~\ref{rem: relaxing linear independence} 
yields:
$$
|T(S)| 
\ge |T(A \cup B)| 
\ge |T(V_1)| \cdot \big( |V_2|+1 \big)
= |T(A)| \cdot \big( |B|+1 \big).
$$
The last step follows from the affine invariance of the capacity. 
\end{proof}

\begin{example}[Capacity of the Boolean cube]		\label{ex: Zuev}
  Let us apply Lemma~\ref{lem: totally separated} to the Boolean cube $H^n$. 
  This cube splits naturally into two totally separated copies of $H^{n-1}$ formed by 
  opposite faces of $H^n$. Using Lemma~\ref{lem: totally separated} and taking logarithms
  of both sides, we get:
  $$
  C(H^n) \ge C(H^{n-1}) + \log_2 |H^{n-1}|
  = C(H^{n-1}) + n-1.
  $$
  By induction, this gives:
  $$
  C(H^n) \ge (n-1) + (n-2) + \cdots + 1 = \frac{n(n-1)}{2}.
  $$
  This recovers the lower bound given in Theorem~\ref{thm: Zuev}.
\end{example}

\section{Capacity of general sets}		\label{s: capacity set lower}	

At the beginning of Section~\ref{s: product sets} we stated that the simple lower bound:
$$
C(S) \ge \log_2|S|+1,
$$
which is valid for any finite set $S \subset \R^n$, 
can be significantly improved 
if we assume that $S$ lies in the Boolean cube $H^n$. 
The following result (restating Theorem~\ref{thm: BTF on general S}) gives such improvement.
 
\begin{theorem}[Capacity of a set]		\label{thm: BTF on general S}
  The capacity of any set $S \subset H^n$ satisfies:
  $$
  C(S) > \frac{1}{16} \log_2^2 |S|.
  $$ 
\end{theorem}

Before we prove this result, let us note that this bound is generally tight for any magnitude of $|S|$, 
up to an absolute constant factor. Indeed, consider the cube $S = H^k$ 
as a subset of $H^n$. Theorem~\ref{thm: BTF on general S} gives: 
$$
C(S) = C(k) \asymp k^2 = \log_2^2|S|,
$$
which matches the bound in Theorem~\ref{thm: Zuev}.

\subsection{A hierarchical decomposition}		\label{s: hierarchical decomposition}

To prove Theorem~\ref{thm: BTF on general S}, we are going to construct
a {\em hierarchical decomposition} of $S$ into totally separated sets.\footnote{We introduced totally separated sets in Section~\ref{s: totally separated}.}
The next lemma defines the decomposition, and 
Lemma~\ref{lem: totally separated} will be used to keep track of the change in capacity at each step. 

\begin{lemma}[A totally separated partition]		\label{lem: totally separated partition}
  Any set $S \subset H^n$ that consists of more than one point
  can be partitioned into two non-empty totally separated subsets $A$ and $B$.
\end{lemma}

\begin{proof}
Choose a pair of distinct points $x,y \in S$. They must differ in at least one coordinate $i$, 
and without loss of generality we may assume that $x_i=0$ and $y_i=1$.
Let the set $A$ consist of all points of $S$ whose $i$-th coordinate equals $0$, 
and $B$ consist of all points of $S$ whose $i$-th coordinate equals $1$.
Then the sets $A$ and $B$ form a partition of $S$ and they are non-empty 
since $x \in A$ and $y \in B$.
\end{proof}

We use the following procedure to decompose $S$ into a tree of totally separated subsets. 
First, let: 
$$
U_0 := S.
$$
If $|U_0| = 1$, stop. Otherwise Lemma~\ref{lem: totally separated partition} gives a partition: 
\begin{equation}	\label{eq: U0 decomposed}
U_0 = U_1 \sqcup V_1, \quad |U_1| \ge |V_1| \ge 1,
\end{equation}
where $U_1$ and $V_1$ are totally separated sets.
If $|U_1| = 1$, stop. Otherwise Lemma~\ref{lem: totally separated partition} gives a partition: 
\begin{equation}	\label{eq: U1 decomposed}
U_1 = U_2 \sqcup V_2, \quad |U_2| \ge |V_2| \ge 1,
\end{equation}
where $U_2$ and $V_2$ are totally separated sets. 
Generally, after $i-1$ partitioning steps, we check if $|U_{i-1}| = 1$ and if so, we stop. 
Otherwise Lemma~\ref{lem: totally separated partition} gives a partition: 
\begin{equation}	\label{eq: Ui decomposed}
U_{i-1} = U_i \sqcup V_i, \quad |U_i| \ge |V_i| \ge 1,
\end{equation}
where $U_i$ and $V_i$ are totally separated sets. 

Since at each step the set $U_i$ becomes strictly smaller, 
the iterative construction must terminate after a finite number $K$ of steps, when we have:
$$
|U_K| = 1.
$$
Firgure~\ref{fig: decomposition} may help to visualize the decomposition process.
\begin{figure}[htp]			
  \centering 
    \includegraphics[width=0.27\textheight]{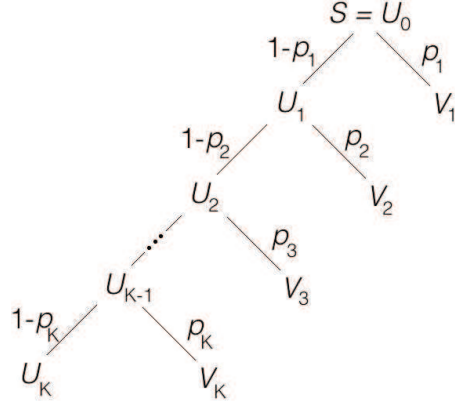} 
    \caption{A hierarchical decomposition of a set $S \subset H^n$ into totally separated subsets.}
  \label{fig: decomposition}
\end{figure}

There are two overlapping situations where a hierarchical decomposition of $S$ automatically 
yields a good lower bound on the capacity of $S$: (1) when the tree is tall, i.e. $K$ is large; and (2) when many ``leaves'' $V_i$ are not too small. The following lemma quantifies this statement.
 
\begin{lemma}[Hierarchical decomposition and capacity]		\label{lem: decomposition capacity}
  In the hierarchical decomposition described above, one has:
  $$
  C(S) > K
  \qquad \text{and} \qquad
  C(S) > \sum_{i=1}^K \log_2 |V_i|.
  $$
\end{lemma}

\begin{proof}
Applying Lemma~\ref{lem: totally separated} for decompositions
\eqref{eq: U0 decomposed} and \eqref{eq: U1 decomposed}, we get: 
$$ 
|T(S)| 
= |T(U_0)| 
\ge |T(U_1)| \cdot \big( |V_1|+1 \big) 
\ge |T(U_2)| \cdot \big( |V_1|+1 \big) \big( |V_2|+1 \big).
$$ 
Continuing in this way, after $K$ steps we get: 
\begin{equation}	\label{eq: CS long product}
|T(S)| \ge 2 \big( |V_1|+1 \big) \big( |V_2|+1 \big) \cdots \big( |V_K|+1 \big),
\end{equation}
since at the last step $|U_K| = 1$ and thus $|T(U_K)| = 2$.
To get the first conclusion of the lemma, note that $|V_i|+1 > 1$ and take the logarithm of both sides 
of \eqref{eq: CS long product}.
To get the second conclusion, note that $|V_i|+1 > |V_i|$ and finish similarly.
\end{proof}

\subsection{Proof of Theorem~\ref{thm: BTF on general S}}
Let:
$$
s \coloneqq \log_2|S|.
$$
If $s \le 16$, the conclusion of the theorem follows from 
the trivial capacity bound in Lemma~\ref{lem: capacity elementary lower}:
$$
C(S) > s \ge \frac{s^2}{16}.
$$
Thus in the rest of the proof we can assume that:  
\begin{equation}	\label{eq: s bound}
s > 16.
\end{equation}

\subsubsection*{Step 1. Stopping criterion.}
Consider the hierarchical decomposition of $S$ constructed in Section~\ref{s: hierarchical decomposition}.
We will need only the initial portion of that tree decomposition, where the sets $U_i$ are still large. 
Specifically, let $k$ be the smallest integer such that:
\begin{equation}	\label{eq: stopping}
|U_k| \le 2^{s/2};
\end{equation}
our argument will focus on the sets $U_i$ and $V_i$ for $i \le k$ only.
Note that:
$$
1 < k \le K.
$$
The upper bound is trivial. To check the lower bound, recall that:
\begin{align*} 
|U_1| 
  &\ge \frac{1}{2} |U_0|   \quad \text{(due to \eqref{eq: U0 decomposed})} \\
  &= \frac{1}{2} |S| = 2^{s-1} > 2^{s/2} \quad \text{(since $s > 16$).}
\end{align*}
The definition of $k$ then yields $k > 1$. 

\subsubsection*{Step 2. Tall trees.}
If $k^2 \ge s^2/16$ then the conclusion of the theorem follows from 
the first bound in Lemma~\ref{lem: decomposition capacity}. Indeed, 
in this case we have:
$$
C(S) > K \ge k \ge \frac{s^2}{16} = \frac{1}{16} \log_2^2 |S|.
$$
Thus, in the rest of the proof we may assume that:
\begin{equation}	\label{eq: k bounds}
1 < k < \frac{s^2}{16}.
\end{equation}

\subsubsection*{Step 3. Decomposition proportions}
Recall that in the hierarchical decomposition \eqref{eq: Ui decomposed}, 
the set $U_{i-1}$ is partitioned into two sets $U_i$ and $V_i$. Let $1-p_i$ and $p_i$ 
denote the proportions of these sets 
(Figure~\ref{fig: decomposition}), i.e. 
\begin{equation}	\label{eq: pi}
|U_i| = (1-p_i) |U_{i-1}| \quad {\rm and} \quad  |V_i| = p_i |U_{i-1}|.
\end{equation}
The condition $|U_i| \ge |V_i| \ge 1$ in \eqref{eq: Ui decomposed} implies that:
$$
0 < p_i \le \frac{1}{2}.
$$
By induction, we have:
$$
|U_i| = |U_0| (1-p_1) (1-p_2) \cdots (1-p_i).
$$
Let us use this identity for $i=k$. 
By the stopping criterion \eqref{eq: stopping}, and since $|U_0| = |S| = 2^s$, 
we have:
\begin{equation}	\label{eq: exp sum pi}
2^{s/2} 
\ge 2^s (1-p_1) (1-p_2) \cdots (1-p_k)
\ge 2^s \, 2^{-2(p_1+\cdots p_k)}.
\end{equation}
To get the last bound we used the numerical inequality $1-x \ge 2^{-2x}$, 
which is valid for all $0 \le x \le 1/2$; we can apply it since $0 < p_i \le 1/2$ for all $i$.
Rearranging the terms in the bound \eqref{eq: exp sum pi} gives:
\begin{equation}	\label{eq: sum pi}
p_1 + \cdots + p_k \ge \frac{s}{4}.
\end{equation}
As a consequence, there must be many $p_i$ that are not too small. 
Specifically, consider the subset of indices $I \subset \{1,\ldots,k\}$ 
defined by:
$$
I \coloneqq \Big\{ i: \; p_i \ge \frac{s}{8k} \Big\}.
$$
We claim that:
\begin{equation}	\label{eq: I}
|I| \ge \frac{s}{4}.
\end{equation}
Indeed, according to \eqref{eq: sum pi}, we have:
$$
\frac{s}{4} 
\le \sum_{i=1}^k p_i
= \sum_{i \in I} p_i + \sum_{i \not\in I} p_i.
$$
There are $|I|$ terms $p_i$ in the first sum, all of which are bounded by $1/2$.
There are at most $k$ terms in the second sum, all of which are bounded by $s/8k$
according to the definition of $I$. Therefore:
$$
\frac{s}{4} \le |I| \cdot \frac{1}{2} + k \cdot \frac{s}{8k}.
$$
Solving this inequality gives $|I| \ge s/4$, as  claimed in \eqref{eq: I}.

\subsubsection*{Step 4. Short trees.}
We are going to use the second bound in Lemma~\ref{lem: decomposition capacity}. 
To apply it effectively, we will first show that all the sets $V_i$ for $i \in I$ are not too small. 
So, fix an $i \in I$ and recall that by the definition of the proportions $p_i$, we have:
$$
|V_i| = p_i |U_{i-1}|.
$$
Since $i \in I$, one has: $p_i \ge 2/8k$. Furthermore considering that $i \le k$, together 
with the definition of the stopping time $k$, yields $|U_{i-1}| \ge 2^{s/2}$. 
Thus:
\begin{equation}	\label{eq: Vi large}
|V_i| \ge \frac{s}{8k} \cdot 2^{s/2} 
\ge \frac{2}{s} \cdot 2^{s/2}
\ge 2^{s/4}.
\end{equation}
In the second bound we use the assumption $k < s^2/16$ from \eqref{eq: k bounds},
and in the last bound we use the assumption $s \ge 16$ from \eqref{eq: s bound}.

Now we are ready to apply the second bound in Lemma~\ref{lem: decomposition capacity}.
It gives in particular:
$$
C(S) > \sum_{i \in I} \log_2 |V_i|.
$$
There are at least $s/4$ terms in this sum due to \eqref{eq: I}, 
each bounded below by $s/4$ according to \eqref{eq: Vi large}. It follows that:
$$
C(S) > \frac{s}{4} \cdot \frac{s}{4} = \frac{s^2}{16} = \frac{1}{16} \log_2^2 |S|.
$$
completing the proof of Theorem~\ref{thm: BTF on general S}.
\qed

\bigskip

Although Theorem~\ref{thm: BTF on general S} gives a bound that is generally tight, 
for many subsets $S \subset H^n$ it can be improved even further. 
We address this phenomenon in Section~\ref{s: enrichment}, 
where we study the {\em enrichment} transformation as a way of increasing the capacity.

\section{Networks with one hidden layer: multiplexing}			\label{s: multiplexing}

Starting from this section, we focus on networks with  at least one hidden layer.
The ultimate goal is to prove the tight lower bound on their capacity stated in Theorem~\ref{thm: main}.
But for now, we begin with a more basic question. For a given input set $S \subset \R^n$, 
can we relate the capacity of the network with one hidden layer $C(S,n,m,1)$  
to the capacity $C(S)=C(S,n,1)$ of the set $S$? 
It is easy to derive a simple upper bound.

\begin{proposition}[The effect of a hidden layer: upper bound]		\label{prop: two layers upper}	
  For any $n,m \ge 4$ and any finite set $S \subset \R^n$, we have:
  $$
  C(S,n,m,1) \le 2C(S)m.
  $$
\end{proposition}

\begin{proof}
The argument is similar to the proof of
Proposition~\ref{prop: multi-channel upper}.
We need to count all functions of the form $\psi \circ \phi$
where $\phi \in T(S,n,m)$ and $\psi \in T(V,m,1) = T(V)$, and where: 
$$
V \coloneqq \im \phi \subset H^m.
$$
The cardinality of the image of $\phi$ is bounded by the cardinality of its domain, so:
$$
|V| = |\im \phi| \le |S|.
$$
There are $|T(S,n,m)|$ functions $\phi$, and for each $\phi$ there are 
$|T(V)|$ functions $\phi$. Thus the total number of compositions $\phi \circ f$  
is:
$$
|T(S,n,m,1)|
\le |T(S,n,m)| \cdot \max_V |T(V)|,
$$
where the maximum is taken over all subsets $V \subset H^m$ 
with cardinality at most $|S|$.
Taking logarithms on both sides gives:
\begin{equation}	\label{eq: CSm1}
C(S,n,m,1) \le C(S,n,m) + \max_{V} C(V).
\end{equation}
Property~\ref{pr: two layers} of Lemma~\ref{lem: capacity basic} gives:
$$
C(S,n,m) = C(S) m.
$$
Furthermore, using the capacity bound from Lemma~\ref{lem: capacity elementary upper},
the assumption on $|V|$, and Lemma~\ref{lem: capacity elementary lower}, wee see that:
$$
C(V) \le m \log_2|V| \le m \log_2|S| \le C(S) m.
$$
Substituting these two bounds in \eqref{eq: CSm1} completes the proof.
\end{proof}

We can interpret Proposition~\ref{prop: two layers upper} as a result that 
compares capacities of single-output and multiple-output networks. 
Indeed, due to part~\ref{pr: two layers} of Lemma~\ref{lem: capacity basic}, 
the bound in Proposition~\ref{prop: two layers upper} states that:
$$
C(S,n,m,1) \le 2C(S,n,m).
$$
What about the converse: can channeling the output through a single node
substantially reduce the capacity of a network? 
In principle, it can. 
Indeed, $C(S,n,m,1)$ is always bounded by $2^{|S|}$, the logarithm of the 
total number of binary functions on $S$, 
while $C(S,n,m) = m C(S)$ is always bounded below by $2m$ 
due to Lemma~\ref{lem: capacity elementary lower}.
Thus, whenever $|S| \ll \log_2 m$, we necessarily have:
$$
C(S,n,m,1) \ll C(S,n,m).
$$ 
Nevertheless, we will now show how to prevent the collapse in capacity 
by modifying $S$ a little -- namely, by adding just $\log_2 m$ bits to the input.

\begin{theorem}[The effect of a hidden layer: lower bound]		\label{thm: two layers general lower}
  Let $S \subset \R^n$ be a finite set.
  Let $m^- \coloneqq \lceil \log_2 m \rceil$
  and $S^+ \coloneqq S \oplus H^{m^-}$. 
  Then: 
  $$
  C(S^+,n+m^-,m,1) \ge C(S,n,m) = C(S) m.
  $$
\end{theorem}

The proof of this theorem is based on a {\em multiplexing} technique,  
which allows one to transmit $m$ output functions through a single channel.
To describe this technique, fix an arbitrary injective map:
$$
\s : \{1,\ldots,m\} \to H^{m^-},
$$
where $m^- \coloneqq \lceil \log_2 m \rceil$.

\begin{lemma}[Multiplexing]		\label{lem: multiplexing}
  Let $S \subset \R^n$ be a finite set.
  Then, for any function $f = (f_1,\ldots,f_m) \in T(S,n,m)$, 
  we can construct a function $f^+ \in T(S^+,n+m^-,m,1)$ such that: 
  \begin{equation}	\label{eq: multiplexing}
  f^+(x \oplus x^-) = f_i(x) 
  \end{equation}
  if $x^- = \s(i)$ for some $i$.
\end{lemma}

Note that the injectivity of $\s$ guarantees 
that there exists at most one $i$ that satisfies \eqref{eq: multiplexing}.

\begin{proof}
Define:
\begin{equation}	\label{eq: fi+}
f_i^+(x \oplus x^-) \coloneqq f_i(x) \wedge (x^- = \s(i)), \quad i=1,\ldots,m
\end{equation}
and:
$$
f^+ \coloneqq f_1^+ \vee \cdots \vee f_m^+.
$$
This definition and the injectivity of $\s$ ensure that \eqref{eq: multiplexing} holds. 
Moreover, each $f_i^+$ is a threshold function according to Lemma~\ref{lem: adding a clause},
i.e. $f_i^+ \in T(S^+)$.
Since the \textsc{or} operation ($\vee$) is also a threshold function,
it follows that: $f^+ \in T(S^+,n+m^-,m,1)$. 
\end{proof}

\begin{proof}[Proof of Theorem~\ref{thm: two layers general lower}]
The Multiplexing Lemma~\ref{lem: multiplexing} implies that
the transformation $f \mapsto f^+$ is an injective map
from $T(S,n,m)$ into $T(S^+,n+m^-,m,1)$.
Indeed, \eqref{eq: multiplexing} allows one to uniquely recover all the threshold functions $f_i$ and thus 
$f = (f_1,\ldots,f_m)$ from $f^+$.
Thus:
$$
|T(S^+,n+m^-,m,1)| \ge |T(S,n,m)|.
$$
Taking logarithms on both sides completes the proof.
\end{proof}

Specializing the result to $S =H^n$, yields 
a tight bound on the capacity of networks with a single hidden layer.  

\begin{corollary}[Capacity of networks with a single hidden layer]		\label{cor: two layers}
  If $n \ge 1.1 \lceil \log_2 m \rceil$, then:
  \begin{equation}	\label{eq: two layers non-asymptotic}
  C(n,m,1) \asymp n^2 m.
  \end{equation}
  Moreover, if $n \to \infty$ and $2 \le \log_2 m \ll n$, then:
  \begin{equation}	\label{eq: two layers asymptotic}
  C(n,m,1) = n^2 m (1+o(1)).
  \end{equation}
\end{corollary}

\begin{proof}
The upper bound in \eqref{eq: two layers non-asymptotic} is a partial case of 
Corollary~\ref{cor: main upper bound 1}).
To prove the asymptotic upper bound in \eqref{eq: two layers asymptotic}, note that 
if $n,m \ge 4$, Proposition~\ref{prop: multi-channel upper} gives
$$
C(n,m,1) \le n^2 m + \min(n,m) m.
$$
Furthermore, we have $\min(n,m) m \ll n^2 m$ if $n \to \infty$. 

To obtain the lower bounds, apply Theorem~\ref{thm: two layers general lower} for $S = H^{n-m^-}$. 
It gives:
\begin{align*} 
C(n,m,1) 
  &\ge C(n-m^-) m \\
  &\gtrsim (n-m^-)^2 m \quad \text{(by Theorem~\ref{thm: Zuev})}\\
  &\gtrsim n^2 m,
\end{align*}
where in the last step we used the assumption that $m^- = \lceil \log_2 m \rceil$.
This proves the first part of the corollary.
The second part follows from the same argument and the assumption that $m^- \ll n$. 
\end{proof}

Figure~\ref{fig: multiplexing} illustrates the multiplexing technique of Lemma~\ref{lem: multiplexing}.
The additional $m^-$ input bits form the vector $x^-$ act as {\em selector} bits. These bits are used to select any one of the $m$ functions $f_1,\ldots,f_m$ to be the final output of the network. Since $m^- \ll m$, the selector is very small and usually does not 
interfere with the capacity count.

\begin{figure}[htp]			
  \centering 
    \includegraphics[width=0.4\textwidth]{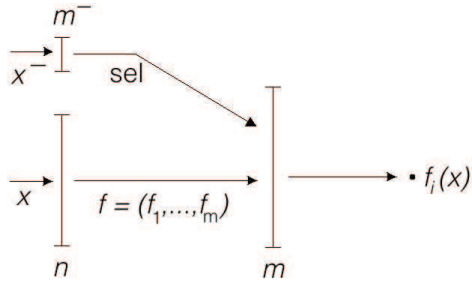} 
    \caption{Multiplexing allows one to transmit any one of the $m$ given functions $f_i$ 
      through a single output channel. The selector bits $x^-$ are used to select which function to transmit.}
  \label{fig: multiplexing}
\end{figure}

\section{Networks with two hidden layers: enrichment}			\label{s: enrichment}

\subsection{Enrichment}

A key recurrent question in this paper is: what is the relation 
between the capacity and cardinality of a general set $S \subset H^n$?
Lemma~\ref{lem: capacity elementary upper} and Theorem~\ref{thm: BTF on general S}
established upper and lower bounds that are generally best possible:
\begin{equation}	\label{eq: capacity vs cardinality}
\log_2^2 |S| \lesssim C(S) \le n \log_2|S|.
\end{equation}
The lower bound, however, is sometimes too weak for practical applications, particularly for
the forthcoming analysis of networks with two hidden layers. 
One may wonder if the capacity of $S$ can be increased by first preprocessing $S$. In particular, 
can we transform $S$ into a set $F(S)$ whose capacity is significantly larger, 
ideally as large as the upper bound in \eqref{eq: capacity vs cardinality} allows?
In doing so, we would like to stay in the category of subsets of the Boolean cube and 
use only transformations $F$
that a neural network can compute. Thus, we require the {\em enrichment map} $F$ 
to be a threshold map $F \in T(n,m)$, i.e. a map from $H^n$ to $H^m$ 
whose all $m$ components are threshold functions.
We address the enrichment problem in the particular case where $S = H^n$, leaving the general 
case of $S \subset H^n$ for future work.

\begin{theorem}[Enrichment]		\label{thm: enrichment}
  Let $n$ and $m$ be positive integers satisfying $n \le m \le 2^{n/2}$.
  There exists an injective linear threhsold map $F \in T(n,m)$ such that:
  $$
  C \big( F(H^n) \big) 
  \asymp nm.
  $$
\end{theorem}

Let us make two remarks before proving this result. First, the map $F$ transforms the cube $S = H^n$ into an ``enriched'' version 
$S' := F(S) \subset H^m$. The enriched set $S'$ has the same cardinality as $S$
and almost the largest possible capacity:
$$
C(S') \asymp nm = m \log_2 |S'|,
$$
which matches the upper bound in \eqref{eq: capacity vs cardinality} in dimension $m$. Second, note also that an upper bound associated with Theorem~\ref{thm: enrichment} holds for any map 
$F: H^n \to H^m$.
This follows straight from Lemma~\ref{lem: capacity elementary upper}.
Indeed, $S' = F(H^n)$ is a subset of $H^m$ and has cardinality $2^n$,
so:
$$
C(S') \le m \log_2 |S'| = mn.
$$
The non-trivial part in Theorem~\ref{thm: enrichment} is the lower bound.
Our construction of $F$ will be based on sparsity considerations.

\subsection{Construction of the enrichment map}		\label{s: construction of enrichment}

Let $k$ be a positive integer and $e_i$ be the canonical basis vectors of $\R^{2^k}$. 
Fix any one-to-one map:
$$
f :\; H^k \to \{e_1,\ldots,e_{2^k}\}.
$$
According to Lemma~\ref{lem: sparsifier}, $f \in T(k,2^k)$.
Define the enrichment map $F: H^n \to H^m$ by applying $f$ to each block of $k$ 
successive coordinates of $x$. 
For $F$ to be well defined, the length $k$ of the blocks must satisfy the equation:
\begin{equation}	\label{eq: nmk}
\frac{n}{k} = \frac{m}{2^k},
\end{equation}
as both sides of the equation determine the number of blocks. 
Assume for now that this equation has an integer solution $k \in [2, n/2]$, and let us
prove the theorem in this `balanced' case.  
The general case will be considered in 
Sections~\ref{s: rich embedding proof general}--\ref{s: full generality}.

For this, we partition a vector $x \in H^n$ into $n/k$ vectors $x_i \in H^k$, 
each containing a block of successive coordinates of length $k$:
$$
x = \bar{x}_1 \oplus \cdots \oplus \bar{x}_{n/k},
$$
and define:
$$
F(x) \coloneqq f(\bar{x}_1) \oplus \cdots \oplus f(\bar{x}_{n/k}).
$$
Since $f$ is a Boolean threshold map, $F$ is a threshold map too,
i.e. $F \in T(n,m)$ as required.

\subsection{Proof of Theorem~\ref{thm: enrichment} in the balanced case}	\label{s: proof balanced}

By construction, the image of $F$ consists of $p = n/k$ copies 
of the image of $f$:
$$
F(H^n) = U \oplus \cdots \oplus U, 
\quad \text{where} \quad
U \coloneqq f(H^k).
$$
Next, recall that the image of $f$ is the set of  canonical vectors of $\R^{2^k}$, i.e. 
$$
U = \{e_1,\ldots,e_{2^k}\}.
$$
Let us apply Corollary~\ref{cor: CSp}.
Since  $U$ is linear independent, $|U| = 2^k \ge 2$,
and $p = n/k \ge 2$ by the assumptions on $k$,
the corollary can be applied. This application gives:
\begin{align*} 
C(F(H^n) 
  &\ge \frac{1}{8} p^2 |U| \log_2|U|
  = \frac{1}{8} \Big( \frac{n}{k} \Big)^2 2^k k \\
  &= \frac{1}{8} n^2 \cdot \frac{2^k}{k}
  = \frac{1}{8} n^2 \cdot \frac{m}{n} \quad \text{(by \eqref{eq: nmk})}\\
  &= \frac{1}{8} n m.
\end{align*}
This proves Theorem~\ref{thm: enrichment} in the special balanced case, where the equation \eqref{eq: nmk} has an integer solution $k \in [2,n/2]$.
Note that the argument so far did not use the assumption $4n \le m \le 2^{n/2}$ of the theorem;
this assumption is used next in order to address the general (unbalanced) case.
\qed

\subsection{Balancing} \label{s: rich embedding proof general}

The following lemma shows how to adjust $n$ and $m$ 
so that the equation \eqref{eq: nmk} has an integer solution $k$. 

\begin{lemma}			\label{lem: n0m0k}
  Let $n \ge 4$ and $m$ be positive integers such that $4n \le m \le 2^{n/2}$. 
  Then there exist integers $n_0 \in [n/2,n]$, $m_0 \in [m/8,m/2]$, and $k \in [2,n/2]$
  such that: 
  \begin{equation}	\label{eq: n0m0k}
  \frac{n_0}{k} = \frac{m_0}{2^k}.		
  \end{equation}
\end{lemma}

\begin{proof}
We claim that:
\begin{equation}	\label{eq: x}
\frac{n}{x} = \frac{m/2}{2^x} \quad \text{for some } x \in \Big[ 2, \frac{n}{2} \Big].
\end{equation}
To show this, consider the function $r: \R^+ \to \R^+$ defined by:
$$
r(x) = \frac{2^x}{x}.
$$
It is easy to check that $r$ increases to infinity on the interval $[2,\infty)$.
Since $r(2) = 2 \le m/(2n)$ by assumption, the intermediate value theorem 
guarantees the existence of a point $x \ge 2$ where 
$r(x) = m/(2n)$. Equivalently, the equation \eqref{eq: x} has a solution $x \in [2,\infty)$.
To give an upper bound on $x$, note that by the assumptions on $n$ and $m$ 
we have:
$$
r(x) = \frac{m}{2n} \le \frac{2^{n/2}}{n/2} = r(n/2).
$$
Since $r$ is an increasing function on the interval $[2,\infty)$ and both $x$ and $n/2$ lie 
in this interval, it follows that $x \le n/2$. This verifies our claim.

Now define: 
$$
k \coloneqq \lceil x \rceil, \quad
n_0 \coloneqq \left\lceil \frac{n}{k} \right\rceil k, \quad
m_0 \coloneqq \left\lceil \frac{n}{k} \right\rceil 2^k.
$$
Then the identity \eqref{eq: n0m0k} obviously holds.
Next, we must check the ranges for $k$, $n_0$ and $m_0$.

By the definition of $k$, we have $k = \lceil x \rceil \ge 2$ since $x \ge 2$ and 
$k = \lceil x \rceil \le x \le n/2$. Thus $k \in [2,n/2]$, as required. 

By the definition of $n_0$, we have:
$n_0 \le (n/k) k = n$ and: 
$$
n_0 \ge \Big( \frac{n}{k}-1 \Big) k = n-k \ge \frac{n}{2},
$$
where the last bound holds since $k = \lceil x \rceil \le x \le n/2$.
Thus $n_0 \in [n/2,n]$, as required.

Finally, by the definition of $m_0$, we have:
\begin{equation}	\label{eq: m0 upper}
m_0 \le \frac{n}{k} 2^k 
\le \frac{n}{x} 2^x 
= \frac{m}{2}
\end{equation}
where the middle bound holds since $r$ increases on the interval $[2,\infty)$, 
both $k$ and $x$ lie in that interval, and $k = \lceil x \rceil \le x$. 
The last bound in \eqref{eq: m0 upper} follows from \eqref{eq: x}.

As for the lower bound on $m_0$, the definition of $m_0$ yields:
\begin{align*} 
m_0 
&\ge \Big( \frac{n}{k} - 1 \Big) 2^k \\
&\ge \Big( \frac{n}{x} - 1 \Big) 2^{x-1} 
	\quad \text{(since $k = \lceil x \rceil$ satisfies $x-1 \le k \le x$)} \\
&\ge \frac{n}{2x} \cdot \frac{1}{2} 2^x
	\quad \text{(since $n/x \ge 2$, which follows from \eqref{eq: x})} \\
&= \frac{m}{8}
	\quad \text{(by the identity \eqref{eq: x} that determines $x$).}
\end{align*}
Thus, in short, $m_0 \in [m/8,m/2]$ and the proof of the lemma is complete.
\end{proof}

\subsection{Proof of Theorem~\ref{thm: enrichment} in full generality}	\label{s: full generality}

Without any loss of generality, we can assume that:
$$
n \ge 4 
\quad \text{and} \quad 4n \le m \le 2^{n/2}.
$$
Indeed, for $n < 4$ the conclusion of the theorem
is trivially true by adjusting the implicit absolute constant factors.
In the range $n \le m \le 4n$, we can use the identity embedding $F$
and get the conclusion from Theorem~\ref{thm: BTF on general S} 
or Theorem~\ref{thm: Zuev}. 

This allows us to apply Lemma~\ref{lem: n0m0k}. 
Let $n_0$, $m_0$ and $k$ be the numbers from the conclusion of that lemma.
Then there exist a map $F' \in T(n_0,m_0)$ such that:
\begin{equation}	\label{eq: CImF'}
C \big( F'(H^{n_0}) \big) \gtrsim n_0 m_0
\end{equation}
This follows from the balanced case of the theorem we proved in 
Sections~\ref{s: construction of enrichment}--\ref{s: proof balanced}, by replacing $n$ and $m$ with $n_0$ and $m_0$ in that argument.

Now extend $F' \in T(n_0,m_0)$ to a map $F \in T(n,m)$ using the identity function. More formally, 
partition each vector $x \in H^n$ as:
$$
x = \bar{x}' \oplus \bar{x}'', 
\quad \text{where } \bar{x}' \in H^{n_0} \text{ and } \bar{x}'' \in H^{n-n_0},
$$
and define $F(x) \in H^m$ by: 
$$
F(x) \coloneqq F'(\bar{x}') \oplus \bar{x}'' \oplus 0.
$$
Here $0 = (0,\ldots,0)$ is a padding vector of zeros, which we add in order to make 
$F(x)$ consist of exactly $m$ coordinates.

We must check that $F(x)$ is well defined.
The vector $F'(\bar{x}')$ consists of $m_0$ coordinates and $\bar{x}''$ consists of $n-n_0$ coordinates. 
In order for the concatenation of these two vectors
to fit in $H^m$, we must have: $m_0 + (n-n_0) \le m$.
This is indeed the case since $m_0 \le m/8$ and $n-n_0 \le n \le m/4$ by Lemma~\ref{lem: n0m0k}.

Since both $F'$ and the identity map are injective threshold maps, 
the map $F$ is an injective threshold map too. By construction, the projection of 
$F(H^n)$ onto the first $m_0$ coordinates equals $F'(H^{n_0})$. Therefore:
$$
C \big( F(H^n) \big) 
  \ge C \big( F'(H^{n_0}) \big) \gtrsim n_0 m_0 \gtrsim n m,
$$
where we used \eqref{eq: CImF'} and Lemma~\ref{lem: n0m0k}. This completes the proof of 
Theorem~\ref{thm: enrichment}.
\qed

\subsection{Capacity of networks with two hidden layers}

As an application of the Enrichment Theorem~\ref{thm: enrichment}, 
we can estimate the capacity of networks with two hidden layers.

\begin{theorem}[Two hidden layers]			\label{thm: three layers}
  If $n \ge 3 \lceil \log_2 m \rceil$, $m \ge 3 \lceil \log_2 p \rceil$,
  and $n \ge 3 \lceil \log_2 p \rceil$, then:
  $$
  C(n,m,p,1) \asymp n^2 m + \min(n, m) mp.
  $$
\end{theorem}

\begin{proof}
The upper bound follows as a special case of Corollary~\ref{cor: main upper bound 1}.
To prove the lower bound, let us first only assume that: 
\begin{equation}	\label{eq: weaker assumptions nmp}
n \ge 2 \lceil \log_2 m \rceil, \quad 
m \ge 2 \lceil \log_2 p \rceil, \quad
n \ge 2 \lceil \log_2 p \rceil.
\end{equation}
Then we can obtain the $n^2 m$ term by comparing 
the $A(n,m,p,1)$ network with the $A(n,m,1)$ network. 
Indeed, we have:
\begin{align*} 
C(n,m,p,1) 
  &\ge C(n,m,1,1) \quad \text{(by monotonicity)} \\
  &\ge C(n,m,1) 
  	\quad \text{(by contractivity, see part~\ref{pr: contractivity} of Lemma~\ref{lem: capacity basic})} \\
  &\gtrsim n^2 m \quad \text{(by Corollary~\ref{cor: two layers}).}
\end{align*}
Next, we consider two cases: $n \ge m$ and $n < m$.

\subsubsection*{Case 1: $n \ge m$}
In this regime, we can compare the two-hidden-layers network with the single-hidden-layer network $A(m,p,1)$. Just like above, using monotonicity, contractivity, and Corollary~\ref{cor: two layers}, 
we get:
$$
C(n,m,p,1) \ge C(m,m,p,1) \ge C(m,p,1) \gtrsim m^2 p.
$$

\subsubsection*{Case 2: $n < m$}
In this regime, we use both enrichment and multiplexing. 
The first assumption in \eqref{eq: weaker assumptions nmp} yields $m \le 2^{n/2}$, 
which allows us to use Theorem~\ref{thm: enrichment}.
Fix an enrichment map $F \in T(n,m)$ whose existence is guaranteed by Theorem~\ref{thm: enrichment}.
Applying part~\ref{pr: replace by image} of Lemma~\ref{lem: capacity basic}, 
for the map $(F \oplus \id)(x \oplus x^-) = F(x) \oplus x^-$
that belongs to the class $T(n+p^-, m+p^-)$ and for $S = H^{n+p^-}$, we obtain:
\begin{align*} 
C(n+p^-,m+p^-,p,1)
&\ge C \big( F(H^n) \oplus H^{p^-}, m+p^-, p,1 \big) \\
  &\ge C \big( F(H^n) \big) p \quad \text{(by Theorem~\ref{thm: two layers general lower})} \\
  &\gtrsim nmp \quad \text{(by Theorem~\ref{thm: enrichment}).}
\end{align*}

\subsubsection*{Putting everything together}
In summary, we showed that $C(n+p^-,m+p^-,p)$ is always bounded below by $n^2 m$,
and is also bounded below by $m^2 p$ if $m < n$, and by $nmp$ if $m \ge n$.
This means that: 
$$
C(n+p^-,m+p^-,p,1) \gtrsim n^2 m + \min(n, m) mp.
$$
The last two assumptions in \eqref{eq: weaker assumptions nmp} state that:
$p^- = \lceil \log_2 p \rceil \le n/2$ and $p^- \le m/2$.
Thus monotonicity gives:
$$
C(\lfloor 3n/2 \rfloor, \lfloor 3m/2 \rfloor, p, 1) \ge C(n+p^-,m+p^-,p,1) \gtrsim n^2 m + \min(n, m) mp.
$$
Recall that we proved this result under the  assumptions \eqref{eq: weaker assumptions nmp}, which are weaker than those in the statement of the theorem. 
Applying this result for $\lfloor 2n/3 \rfloor$ instead of $n$, 
and for $\lfloor 2m/3 \rfloor$ instead of $m$, completes the proof.
\end{proof}

\begin{theorem}[Two hidden layers, multiple-outputs]			\label{thm: three layers multi-channel}
  If $n \ge 2 \lceil \log_2 m \rceil$ and $m \ge 2 \lceil \log_2 p \rceil$, then:
  $$
  C(n,m,p) \asymp n^2 m + \min(n, m) mp.
  $$
\end{theorem}

\begin{proof}
The upper bound is a partial case of Corollary~\ref{cor: main upper bound}.
For the lower bound, we can essentially repeat the proof of Theorem~\ref{thm: three layers}
except for the multiplexing in the last step, which is not needed in this case. Instead, we can 
just use part~\ref{pr: replace by image} of Lemma~\ref{lem: capacity basic}  
followed by the enrichment Theorem~\ref{thm: enrichment} and get:
$$
C(n,m,p)
\ge C \big( F(H^n),m,p \big)
\gtrsim nmp.
$$
The proof is complete.
\end{proof}

\section{Networks with arbitrarily many layers: stacking}		\label{s: stacking}

Now we extend the capacity lower bounds to feedforward networks with arbitrarily many layers, 
thus completing the proof of the main result (Theorem~\ref{thm: main}). 
Denote:
$$
\bar{n}_k := \min(n_1,\ldots,n_k).
$$
Let us handle networks with three hidden layers first.

\begin{lemma}[Three hidden layers]		\label{lem: four layers}
  Let $n_j \ge 3 \lceil \log_2 n_k \rceil$ for all $1 \le j < k \le 4$. Then: 
  $$
  C(n_1,n_2,n_3,n_4,1) \asymp n_1^2 n_2 + \bar{n}_2 n_2 n_3 + \bar{n}_3 n_3 n_4.
  $$
\end{lemma}

\begin{proof}
The upper bound is a special case of Corollary~\ref{cor: main upper bound 1}.
As for the lower bound, monotonicity, contractivity (Lemma~\ref{lem: capacity basic}), 
and Theorem~\ref{thm: three layers} yield:
$$
C(n_1,n_2,n_3,n_4,1)
\ge C(n_1, n_2, n_3,1) 
\gtrsim n_1^2 n_2 + \bar{n}_2 n_2 n_3
$$
and also:
$$
C(n_1,n_2,n_3,n_4,1)
\ge C(\bar{n}_2, \bar{n}_2, n_3,n_4,1) 
\ge C(\bar{n}_2, n_3,n_4,1) 
\gtrsim \bar{n}_3 n_3 n_4.
$$
Combining the two lower bounds, we conclude that:
$$
C(n_1,n_2,n_3,n_4,1)
\gtrsim \max \left( n_1^2 n_2 + \bar{n}_2 n_2 n_3, \bar{n}_3 n_3 n_4 \right)
\ge \frac{1}{2} \big( n_1^2 n_2 + \bar{n}_2 n_2 n_3 + \bar{n}_3 n_3 n_4 \big). 
$$
The proof is complete.
\end{proof}

\subsection{Stacking}
In principle, networks with arbitrarily many layers can be handled by a similar argument. 
However, instead of producing the sum over the layers claimed by Theorem~\ref{thm: main}, this argument will only produce the maximum over the layers. 
The maximum can be replaced with the sum by paying a factor of $1/L$, which is weaker 
than the constant factor claimed in Theorem~\ref{thm: main}. Thus, to overcome this limitation, we develop a {\em stacking} technique and prove the following.

\begin{lemma}[Four and more hidden layers]	\label{lem: five and more layers}
  Assume that $L \ge 5$ and $n_j \ge 3 \lceil \log_2 (Ln_k) \rceil$ for all $1 \le j < k \le L$. Then:
  $$
  C(6n_1,\ldots,6n_L,1) \gtrsim \sum_{k=1}^{L-1} \bar{n}_k n_k n_{k+1}.
  $$
\end{lemma}

\begin{proof}
\begin{figure}[htp]			
  \centering 
    \includegraphics[width=\textwidth]{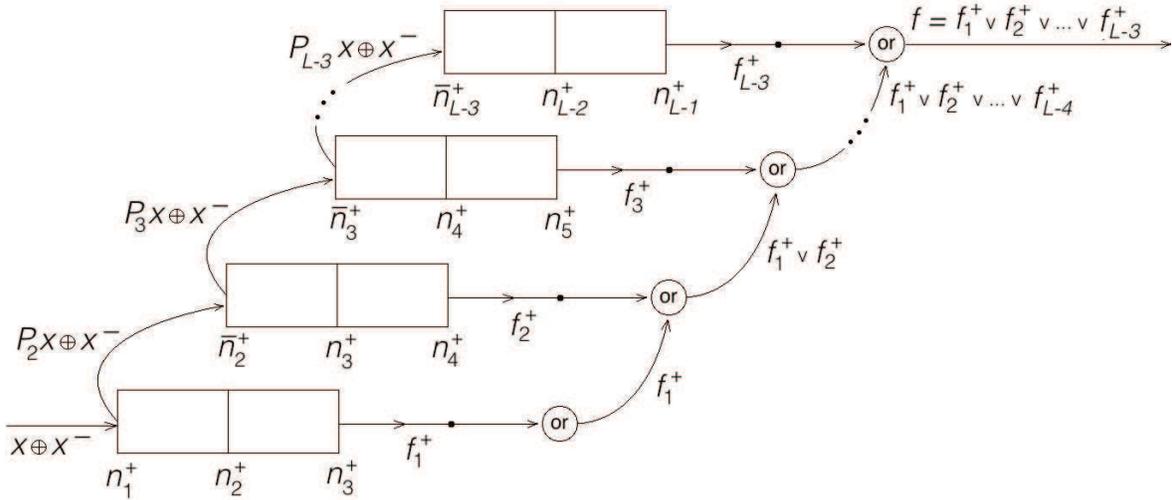} 
    \caption{A network with almost largest possible capacity can be constructed
      by stacking three-layer networks.}
  \label{fig: stacking}
\end{figure}
To prove the lemma, we will compare the network $A(6n_1,\ldots,6n_L,1)$
with a smaller network, which we construct by ``stacking'' $L-3$ three-layer modules, 
and doing multiplexing in each one of them.

\subsubsection*{Step 1. Construction of the network}
Fix an arbitrary injective map
\begin{equation}	\label{eq: L-}
\eta: \{1,\ldots,L-3\} \to H^{L-}
\quad \text{where} \quad
L^- \coloneqq \lceil \log_2 (L-3) \rceil.
\end{equation}
Consider arbitrary functions
$$
f_k \in T(\bar{n}_k, n_{k+1}, n_{k+2},1), 
\quad k = 1,\ldots,L-3,
$$
Lemma~\ref{lem: adding a clause} states that the function
\begin{equation}	\label{eq: fk+}
f_k^+(x \oplus x^-) \coloneqq f_k(x)  \wedge \left( x^- = \eta(k) \right),
\quad x \in H^{\bar{n}_k}, 
\quad x^- \in H^{L^-},
\end{equation}
belongs to $T(\bar{n}_k^+, n_{k+1}^+, n_{k+2}^+,1)$, where we let:
\begin{equation}	\label{eq: nk+}
n_k^+ \coloneqq n_k + L^- 
\quad \text{and} \quad 
\bar{n}_k^+ \coloneqq \bar{n}_k + L^-.
\end{equation}

Now connect three-layer modules $A(\bar{n}_k^+, n_{k+1}^+, n_{k+2}^+,1)$, 
$k = 1,\ldots,L-3$, as shown in Figure~\ref{fig: stacking}.
In that figure, $P_k$ denotes the coordinate projection onto $\R^{\bar{n}_k}$ that 
retains the first $\bar{n}_k$ coordinates of a vector. 
Given an input $x \oplus x^-$, the first module computes $f_1^+(x \oplus x^-)$ in layer $4$, 
and it passes $P_2 x \oplus x^-$ to layer $2$ as the input to the second module. 
The second module computes $f_2^+(P_2 x \oplus x^-)$ in layer $5$, then takes the `or' 
with the output of the second module in layer $6$, thereby computing 
$f_1^+(x \oplus x^-) \vee f_2^+(P_2 x \oplus x^-)$;
it also passes $P_3 x \oplus x^-$ to layer $3$ as the input to the third module, etc.
Continuing in this way, we see that the network ultimately computes and outputs the function: 
\begin{equation}	\label{eq: f stacking}
f(x \oplus x^-) \coloneqq f_1^+(x \oplus x^-) \vee f_2^+(P_2 x \oplus x^-)
\vee \cdots \vee f_{L-3}^+(P_{L-3} x \oplus x^-).
\end{equation}

\subsubsection*{Step 2. Estimating the capacity of the network using the capacities of modules}
Now that we described the architecture, let us estimate 
how many Boolean functions $f$ the architecture can compute. 
Let us denote the set of all such computable functions $f^+$ by $T$. 
By definition of the functions $f_k^+$ in \eqref{eq: fk+} and $f$ in \eqref{eq: f stacking}, 
we have:
\begin{equation}	\label{eq: f+fk}
f(x \oplus x^-) = f_k(P_k x)
\end{equation}
if $x^- = \eta(k)$ for some\footnote{The injectivity of $\eta$ guarantees 
that there exists at most one $k$ that satisfies \eqref{eq: f+fk}.} $k$.
This implies that the map $(f_1,\ldots,f_{L-3}) \mapsto f^+$ is an injective transformation 
from $\prod_{k=1}^{L-3} T(\bar{n}_k, n_{k+1}, n_{k+2},1)$ to $T$.
Indeed, Equation \eqref{eq: f+fk} allows one to uniquely recover all $f_k$ and thus 
$f = (f_1,\ldots,f_{L-3})$ from $f^+$.)
Therefore:
$$
|T| \ge \prod_{k=1}^{L-3} |T(\bar{n}_k, n_{k+1}, n_{k+2},1)|.
$$
The right hand side can be estimated using the tight bounds on the capacity of three-layer networks
from Theorem~\ref{thm: three layers}. Note that the conditions on $n_k$ guarantee 
that the assumptions of Theorem~\ref{thm: three layers} are satisfied. We obtain:
\begin{align} 
\log_2 |T| 
  &\ge \sum_{k=1}^{L-3} C(\bar{n}_k, n_{k+1}, n_{k+2},1) \nonumber\\
  &\gtrsim (n_1^2 n_2 + \bar{n}_2 n_2 n_3) + \bar{n}_3 n_3 n_4 + \bar{n}_4 n_4 n_5 
  	+ \cdots + \bar{n}_{L-2} n_{L-2} n_{L-1} \nonumber\\
  &= \sum_{k=1}^{L-2} \bar{n}_k n_k n_{k+1}.		\label{eq: CC lower}
\end{align}

\subsubsection*{Step 3. Counting nodes}
As is evident from Figure~\ref{fig: stacking}, the 
overall architecture has $L$ layers of units (not counting the output). The number of 
nodes in the $k$-th layer of units, $k=1,\ldots, L-1$, is bounded by:
\begin{align*} 
2 + 2n_k^+ + \bar{n}_k^+
  &\le 2 + 3n_k^+ \\
  &\le 2 + 3 \big( n_k + \lceil \log_2 L \rceil \big)
  	\quad \text{(by definition of $n_k^+$ in \eqref{eq: nk+} and $L^-$ in \eqref{eq: L-})} \\
  &\le 6 n_k 
  	\quad \text{(by the assumptions on $n_k$).}
\end{align*}
Hence, by monotonicity:
$$
\log_2 |T| \le C(6n_1,\ldots,6n_L,1).
$$
Combining this with the lower bound \eqref{eq: CC lower}, we conclude that:
\begin{equation}	\label{eq: all but last term}
C(6n_1,\ldots,6n_L,1) \gtrsim \sum_{k=1}^{L-2} \bar{n}_k n_k n_{k+1}.
\end{equation}

\subsubsection*{Step 4. Adding one term to the sum}
To complete the proof, we just need to add one last term to this sum.
We can get it by comparison with a three-layer network $A(\bar{n}_{L-2},n_{L-1},n_L,1)$.
Indeed, monotonicity, contractivity (Lemma~\ref{lem: capacity basic}), 
and Theorem~\ref{thm: three layers} give:
$$
C(6n_1,\ldots,6n_L,1) 
\ge C(\bar{n}_{L-2},\ldots,\bar{n}_{L-2},n_{L-1},n_L,1) 
\ge C(\bar{n}_{L-2},n_{L-1},n_L,1)
\gtrsim \bar{n}_{L-1} n_{L-1} n_L.
$$
Combining this with \eqref{eq: all but last term}, we conclude that:
$$
C(6n_1,\ldots,6n_L,1) 
  \gtrsim \max \left( \sum_{k=1}^{L-2} \bar{n}_k n_k n_{k+1}, \bar{n}_{L-1} n_{L-1} n_L \right) 
  \ge \frac{1}{2} \sum_{k=1}^{L-1} \bar{n}_k n_k n_{k+1}.
$$
This completes the proof of the Lemma.
\end{proof}

\subsection{The lower bound in Theorem~\ref{thm: main}}		\label{s: lower bound}
Now we prove a partial case of Theorem~\ref{thm: main} for networks with 
a single output node:

\begin{theorem}			\label{thm: single-channel}
  Under the conditions of Theorem~\ref{thm: main}, we have:
  $$
  C(n_1,\ldots,n_L,1) 
  \asymp \sum_{k=1}^{L-1} \bar{n}_k n_k n_{k+1}.
  $$
\end{theorem}

\begin{proof}
We already proved the upper bound on the capacity in Corollary~\ref{cor: main upper bound 1}.
The lower bound follows from Corollary~\ref{cor: two layers} for a single hidden layer,
Theorem~\ref{thm: three layers} for two hidden layers, Lemma~\ref{lem: four layers} for three hidden layers,
and Lemma~\ref{lem: five and more layers} for four and more layers applied\footnote{Precisely, the assumptions of Theorem~\ref{thm: main} yield:
	$n_j \ge 18 \log_2 (L n_k) \ge 18 \lceil \log_2 (Ln_k/6) \rceil$.
	Dividing both sides by $6$ and taking the integer part, we get: 
	$\lfloor n_j/6 \rfloor \ge 3 \lceil \log_2 (Ln_k/6) \rceil \ge 3 \lceil \log_2 (L \lfloor n_k/6 \rfloor) \rceil$.
	This means that Lemma~\ref{lem: five and more layers} can indeed be applied 
	using $\lfloor n_k/6 \rfloor$ instead of $n_k$.} 
using $\lfloor n_k/6 \rfloor$ instead of $n_k$. 
\end{proof}

Finally, we are ready to complete the proof of the main result:

\begin{proof}[Proof of Theorem~\ref{thm: main}]
The upper bound was already proven in Proposition~\ref{prop: multi-channel upper}.
It remains to prove the lower bound.
For $L=2$, the result follows from Theorem~\ref{thm: Zuev}, which gives:
$$
C(n_1,n_2) = C(n_1) n_2 \gtrsim n_1^2 n_2.
$$ 
Now let $L \ge 3$. Monotonicity and Theorem~\ref{thm: single-channel} yield:
$$
C(n_1,\ldots,n_L) 
\ge C(n_1,\ldots,n_{L-1},1) 
\gtrsim \sum_{k=1}^{L-2} \bar{n}_k n_k n_{k+1}.
$$
To complete the proof, we need to add just one last term to this sum.
We can get it by comparison with a three-layer network $A(\bar{n}_{L-2},n_{L-1},n_L)$.
Indeed, monotonicity, contractivity (Lemma~\ref{lem: capacity basic}), 
and Theorem~\ref{thm: three layers multi-channel} give:
$$
C(n_1,\ldots,n_L) 
\ge C(\bar{n}_{L-2},\ldots,\bar{n}_{L-2},n_{L-1},n_L) 
\ge C(\bar{n}_{L-2},n_{L-1},n_L)
\gtrsim \bar{n}_{L-1} n_{L-1} n_L.
$$
Combining this with \eqref{eq: all but last term}, we conclude that:
$$
C(n_1,\ldots,n_L) 
  \gtrsim \max \left( \sum_{k=1}^{L-2} \bar{n}_k n_k n_{k+1}, \bar{n}_{L-1} n_{L-1} n_L \right) 
  \ge \frac{1}{2} \sum_{k=1}^{L-1} \bar{n}_k n_k n_{k+1}.
$$
This completes the proof.
\end{proof}

\subsection{Why are rapidly expanding networks excluded?}		\label{s: rapidly expanding}
We stated Theorem~\ref{thm: main} under the assumption that the network is not 
expanding too rapidly, as quantified by requiring that:
\begin{equation}	\label{eq: condition on the number of nodes}
n_j \ge 18 \log_2 (L n_k) \quad \text{for all } j \le k.
\end{equation}
It is worth noting that this requirement is almost optimal. 
To see this, note first that the number of all Boolean functions 
on $H^{n_1}$ is $2^{2^{n_1}}$.
This yields the trivial upper bound:
$$
C(n_1,\ldots,n_L,1) \le 2^{n_1}.
$$
Combining it with the lower bound given by Theorem~\ref{thm: main} (and Corollary~\ref{cor: effect of output node}), we get
$$
2^{n_1} \gtrsim \sum_{k=1}^{L-1} \min(n_1,\ldots,n_k) n_k n_{k+1} 
\ge \sum_{k=2}^L n_k.
$$
Thus, in order for Theorem~\ref{thm: main} to hold, we must have:
$$
n_1 \gtrsim \log_2 \left( \sum_{k=2}^L n_k \right).
$$
In particular, if all $n_k$ for $k \ge 2$ are of the same order (e.g. equal to each other), we must have:
$$
n_1 \gtrsim \log_2 (L n_k).
$$
This shows that the condition \eqref{eq: condition on the number of nodes}
can not be removed and that it has an almost optimal form.

\subsection{Relaxing the assumption on the number of nodes}				\label{s: smaller top layers}
Although the assumption $n_j \gtrsim \log_2 (L n_k)$ in Theorem~\ref{thm: main} 
is almost optimal, it can still be slightly improved in order to accommodate small top layers.
Specifically, with a little more work, it can be relaxed to: 
$$
n_j \gtrsim \log_2 \big( (L-k+1)n_k \big) \quad \text{for all } j \le k.
$$ 
This relaxed condition can be useful since it allows for very small top layers.

The idea behind the relaxed condition is that in the proof of Lemma~\ref{lem: five and more layers},
it is not necessary to transmit all bits of $x^-$ to the top layer. 
Indeed, choose $\eta(k)$ to be the binary representation of the number $L-3-k$. 
Thus, the first bit of $\eta(k)$ is $0$ if $k$ it is in the uppet half of the layers, 
the first two bits are $00$ if $k$ is in the upper quarter, 
the first three bits are $000$ if $k$ is in the upper eighth, etc.
Now, we can drop the first bit of $x^-$ when we pass it between the modules 
in the upper half of the layers (i.e. for $k \ge (L-3)/2$); instead of verifying the clause
$\eta(k) = x^-$, we verify the equivalent clause $Q_1 \eta(k) = Q_1 x^-$, where
$Q_1$ is the coordinate projection that drops the first bit. 
Similarly, we can drop the second bit of $x^-$ in the upper quarter of the layers, etc. 
Thus, the length of the portion of $x^-$ passed to the $k$-th layer 
is approximately $\log_2(L-k+1)$ instead of the full length, i.e. $L^- = \log_2(L)$. 
The rest of the proof is unchanged.

\subsection{Restricted capacity}			\label{s: restricted capacity lower}

In this section we extend Theorem~\ref{thm: main} to the case where the input 
to the network are not all possible binary vectors, but rather lie in a subset $S \subset H^{n_1}$. We introduced this restricted version of capacity 
in Section~\ref{s: def capacity} and denoted it by:
$$
C(S,n_1,n_2,\ldots,n_L).
$$
We proved an upper bound on $C(S,n_1,n_2,\ldots,n_L)$ in Proposition~\ref{prop: multi-channel upper}.
Now we will complement it with a lower bound. 
The notion of VC-dimension (see e.g. \cite[Section~8.3]{vershynin2018high}) 
allows us to reduce the problem of restricted capacity to the case of unrestricted capacity.

\begin{lemma}[Restricted vs. unrestricted capacity]		\label{lem: restricted vs unrestricted}
  Consider a subset $S \subset H^{n_1}$. 
  Then, for any number of layers $L \ge 2$ 
  and any number of nodes $n_2,\ldots,n_L$ in each layer, we have:
  $$
  C(S,n_1,n_2,\ldots,n_L) \ge C(H^d,n_1,n_2,\ldots,n_L)=C(d,n_2,\ldots,n_L),
  $$
  where $d$ is the VC-dimension of $S$. 
\end{lemma}

\begin{proof}
By the definition of VC-dimension, there exists a subset of indices $I \subset \{1,\ldots,n_1\}$ 
of cardinality $|I|=d$ that is shattered by $S$. This means that: 
$$
P_I S = H^I
$$
where $P_I: H^{n_1} \to H^I$ is the coordinate projection that 
retains the coordinates in $I$ and drops the coordinates outside of $I$.
By excluding the input nodes outside $I$, one  immediately obtains: 
$$
C(S,n_1,n_2,\ldots,n_L) 
\ge C(P_I S,n_1,n_2,\ldots,n_L) 
= C(H^d,n_2,\ldots,n_L,1)=C(d,n_2,\ldots,n_L).
$$
The proof is complete.
\end{proof}

Combining this bound with the Sauer-Shelah Lemma, we obtain the following:

\begin{proposition}[Restricted capacity: a lower bound]	\label{prop: restricted vs unrestricted}
  Consider a subset $S \subset H^{n_1}$ such that $|S| \le 2^{n}$.
  Then, for any number of layers $L \ge 2$ 
  and any number of nodes $n_2,\ldots,n_L$ in each layer, 
  there exists an integer $d$ such that: 
  $$
  d \gtrsim \frac{n}{\log_2(en_1/n)}
  $$
  and:
  $$
  C(S,n_1,n_2,\ldots,n_L) \ge C(H^d,n_1,n_2,\ldots,n_L)=C(d,n_1, \ldots,n_L).
  $$
\end{proposition}

\begin{proof}
The Sauer-Shelah Lemma (see e.g. \cite[Section~8.3.3]{vershynin2018high}) gives the upper bound:
$$
|S| \le \sum_{k=0}^d \binom{n_1}{k} \le \Big( \frac{en_1}{d} \Big)^d,
$$
where $d$ is the VC-dimension of $S$. On the other hand, we have the lower bound
$|S| \le 2^{n}$. Combining the two bounds and taking logarithms, we get:
$$
n \le d \log_2 \Big( \frac{en_1}{d} \Big).
$$
An elementary computation then yields:
$$
d \gtrsim \frac{n}{\log(en_1/n)}.
$$
An application of Lemma~\ref{lem: restricted vs unrestricted} completes the proof.
\end{proof}

Combining Proposition~\ref{prop: restricted vs unrestricted} with
the capacity formula for $C(d,n_2,\ldots,n_L)$ given by  Theorem~\ref{thm: main}, 
we can obtain a general lower bound on the restricted capacity in terms of the cardinality of $S$.

\begin{remark}[Tightness]
  The bound in Proposition~\ref{prop: restricted vs unrestricted} is generally best possible
  up to a logarithmic factor. Indeed, if $S = H^d$ and $d = n_1$ then: 
  $$
  C(S,n_1,n_2,\ldots,n_L) = C(d,n_2,\ldots,n_L).
  $$
\end{remark}

\section{Extremal capacity}		\label{s: max capacity}

The capacity formula in Theorem~\ref{thm: main} is particularly useful
when one wants to maximize the capacity of a network under some natural constraints. 
For example, if we {\em fix the number of parameters} of a network, which is essentially the same as fixing the number of edges, Corollary~\ref{cor: max capacity given parameters}
states that any monotonically expansive network approximately maximizes capacity.
In this section, we consider what happens if instead we {\em fix the number of nodes}
and, possibly, also the number of nodes in the input layer.

We will use the symbols $\approx$ 
for identities that hold up to an $1+o(1)$ factor, that is $a_n \approx b_n$ means that $a_n = (1+o(1)) b_n$ as $n \to \infty$.
As before, we continue to use the symbols $\asymp$ and $\lesssim$ 
for identities and inequalities that hold up to an absolute constant factor.

\subsection{Fixing the number of nodes}

It turns out that a network with a given number of nodes that asymptotically maximizes capacity 
is {\em shallow}. Specifically, the optimal network has just one hidden layer, 
which is half the size of the input layer:

\begin{theorem} 		\label{thm: max capacity given nodes}
  Let $L \ge 2$ and $n_1,\ldots,n_{L-1} \ge 4$, $n_L \ge 1$.  
  Let $N \coloneqq n_1 + \cdots + n_L$ denote the total number of nodes. 
  Then:
  $$ 
  C(n_1,\ldots,n_L) 
    \le \frac{4}{9} N^3
    \approx C \Big( \frac{2N}{3}, \frac{N}{3} \Big)
  $$
  as $N \to \infty$.
\end{theorem}

We shall first prove a version of Theorem~\ref{thm: max capacity given nodes} 
for the {\em estimated} capacity:
\begin{equation}	\label{eq: estimated capacity}
\widehat{C}(n_1,\ldots,n_L) 
\coloneqq \sum_{k=1}^{L-1} \min(n_1,\ldots,n_k) n_k n_{k+1}, 
\end{equation}
and then replace the estimated capacity.
The following lemma yields a general recipe to increase the (estimated) capacity of any network, by moving all nodes from layer $3$ and up into the input layer.

\begin{lemma}[Move nodes out of upper layers to increase capacity]	\label{lem: move nodes}
  Let $L \ge 3$. Then: 
  \begin{equation}	\label{eq: move nodes}
    \widehat{C}(n_1,\ldots,n_L) 
    \le \widehat{C} \Big( n_1 + \sum_{k=3}^L n_k, n_2 \Big).
  \end{equation}
\end{lemma}

\begin{proof}
Let us fist handle the case $L=3$, where we have to show that: 
\begin{equation}	\label{eq: move nodes L=3}
\widehat{C}(n,m,p) \le \widehat{C}(n+p,m).
\end{equation}
The definition of the estimated capacity \eqref{eq: estimated capacity} yields:
$$
\widehat{C}(n,m,p) 
= n^2 m + \min(n,m)mp
\le (n^2+np)m.
$$
In the last step we used that $\min(n,m) \le n$.
On the other hand, the same definition yields:
$$
\widehat{C}(n+p,m) 
= (n+p)^2 m
\ge (n^2 + 2np) m.
$$
Hence \eqref{eq: move nodes L=3} is evident.

Next, let $L \ge 4$.
Combining the definition of the estimated capacity \eqref{eq: estimated capacity} 
with the fact that $\min(n_1,n_2) \le n_1$, $\min(n_1,\ldots,n_k) \le n_2$, 
we obtain:
$$
\widehat{C}(n_1,\ldots,n_L) 
  \le n_1^2 n_2 + n_1 n_2 n_3 + \sum_{k=3}^{L-1} n_2 n_k n_{k+1} 
  = \Big( n_1^2 + n_1 n_3 + \sum_{k=3}^{L-1} n_k n_{k+1} \Big) n_2.
$$
On the other hand, using the same definition and expanding the square, we get:
$$
\widehat{C} \Big( n_1 + \sum_{k=3}^L n_k, n_2 \Big)
= \Big( n_1 + \sum_{k=3}^L n_k \Big)^2 n_2
\ge \Big( n_1^2 + 2n_1 n_3 + 2 \sum_{k=3}^{L-1} n_k n_{k+1} \Big) n_2.
$$
Hence \eqref{eq: move nodes} is evident.
\end{proof}

Armed with the recipe given in Lemma~\ref{lem: move nodes}, 
we can easily maximize the (estimated) capacity over all networks with two layers
and a given number of nodes.

\begin{lemma}[The most capable network with two layers]		\label{lem: most capable two layers}
  Let $N = n+m$. Then:
  $$
  \widehat{C}(n,m) = n^2 m 
  \le \widehat{C} \Big( \frac{2N}{3}, \frac{N}{3} \Big)  
  = \frac{4}{9} N^3.
  $$
\end{lemma}

\begin{proof}
The maximum of $n^2 m = n^2(N-n)$ is attained for $n = 2N/3$. 
\end{proof}

Combining Lemmas~\ref{lem: move nodes} and \ref{lem: most capable two layers},
we obtain a version of Theorem~\ref{thm: max capacity given nodes} for the estimated capacity: 
\begin{equation}	\label{eq: max estimated capacity given nodes}
\widehat{C} (n_1,\ldots,n_L) 
  \le \widehat{C} \Big( n_1 + \sum_{k=3}^L n_k, n_2 \Big) 
  \le \widehat{C} \Big( \frac{2N}{3}, \frac{N}{3} \Big) 
  = \frac{4}{9} N^3.
\end{equation}

\begin{proof}[Proof of Theorem~\ref{thm: max capacity given nodes}]
Because the capacity is bounded by the estimated capacity 
(Proposition~\ref{prop: multi-channel upper}), and using 
\eqref{eq: max estimated capacity given nodes}, we get:
$$
C (n_1,\ldots,n_L) 
  \le \widehat{C} (n_1,\ldots,n_L) 
  \le \widehat{C} \Big( \frac{2N}{3}, \frac{N}{3} \Big) 
  = \frac{4}{9} N^3.
$$
Furthermore, Theorem~\ref{thm: Zuev} implies that:
$$
C \Big( \frac{2N}{3}, \frac{N}{3} \Big) 
= C \big( H^{2N/3} \big) \frac{N}{3}
\approx \Big( \frac{2N}{3} \Big)^2 \frac{N}{3} 
= \frac{4}{9} N^3
$$
as $N \to \infty$.
This completes the proof.
\end{proof}

\subsection{Fixing both the total number of nodes and the size of the input layer}

In many applications, the input layer is fixed and can not be optimized.
In such situations, it makes sense to maximize capacity of networks with a given total
number of nodes $N$, as well as a given number of nodes $n_1$ in the input layer.
While here we focus on the case where the total number of neurons and the size of the input layer are fixed, similar results are obtained also for the case where in addition the size of the output layer is fixed.

It turns out that a network that maximizes capacity under these constraint is again shallow.
If $n_1 \le N/2$, the optimal network has two hidden layers, the first having $n_1$ more 
nodes than the second. 
If $n_1 \ge N/2$, such architecture is impossible; the optimal network has just one hidden layer.
The following theorem makes this precise.

\begin{theorem}		\label{thm: max capacity given nodes and input}
  Let $L \ge 2$. 
  Assume that the total number of nodes is $N \coloneqq n_1 + \cdots + n_L$. 
  Then, the following holds if $n_1 \to \infty$ and $\log N \ll n_1$.
  \begin{enumerate}[\quad 1.]
    \item If $n_1 \le N/2$ then: 
      $$
      C(n_1,\ldots,n_L) 
      \le \frac{n_1 N^2}{4}
      \asymp C \Big( n_1, \frac{N}{2}, \frac{N}{2}-n_1 \Big).
      $$
    \item If $n_1 \ge N/2$ then: 
      $$
      C(n_1,\ldots,n_L) 
      \le n_1^2(n_2+\cdots+n_L)
      \approx C(n_1, n_2+\cdots+n_L).
      $$    
  \end{enumerate}
\end{theorem}

As in the previous section, we first prove a version of 
Theorem~\ref{thm: max capacity given nodes and input}
for the estimated capacity $\widehat{C}(n_1,\ldots,n_L,1)$ defined in \eqref{eq: estimated capacity}.
The following elementary fact will be helpful in our analysis.

\begin{lemma}		\label{lem: quadratic form}
  For any $L \ge 2$, and any positive real numbers $x_1,\ldots,x_L$, we have:
  $$
  \Big( \sum_{k=1}^L x_k \Big)^2 \ge 4 \sum_{k=1}^{L-1} x_k x_{k+1}.
  $$
\end{lemma}

\begin{proof}
Consider the difference:
\begin{equation}	\label{eq: two sums squared}
\Big( \sum_{k=1}^L x_k \Big)^2 - \Big( \sum_{k=1}^L (-1)^k x_k \Big)^2
  = 2 \sum_{i,j=1}^L \Big( 1 - (-1)^i (-1)^j \Big) x_i x_j 
  = 4 \sum_{i,j \in \OO} x_i x_j 
\end{equation}
where $\OO \subset \{1,\ldots,L\}^2$ is the set of pairs $(i,j)$ such that 
either $i$ is even and $j$ is odd, or $i$ is odd and $j$ is even.
In particular, $\OO$ contains all pairs of the form $(k,k+1)$. 
Since all the terms $x_i x_j$ of the sum are positive, this yields:
$$
\sum_{i,j \in \OO} x_i x_j
\ge \sum_{k=1}^{L-1} x_k x_{k+1}.
$$
Combining this with \eqref{eq: two sums squared}, we conclude that:
$$
\Big( \sum_{k=1}^L x_k \Big)^2 - \Big( \sum_{k=1}^L (-1)^k x_k \Big)^2
\ge 4 \sum_{k=1}^{L-1} x_k x_{k+1}.
$$
This yields the conclusion of the lemma.
\end{proof}

We are ready to prove the ``estimated'' version the first part of 
Theorem~\ref{thm: max capacity given nodes and input}.

\begin{lemma}[Small input layer]		\label{lem: small input layer}
  If $n_1 \le N/2$, then: 
  $$
  \widehat{C}(n_1,\ldots,n_L) 
  \le \frac{n_1 N^2}{4}
  = \widehat{C} \Big( n_1, \frac{N}{2}, \frac{N}{2}-n_1 \Big). 
  $$
\end{lemma}

\begin{proof}
On one hand, the definition \eqref{eq: estimated capacity} of the estimated capacity yields:
$$
\widehat{C} \Big( n_1, \frac{N}{2}, \frac{N}{2}-n_1 \Big)
= n_1^2 \cdot \frac{N}{2} + \min \Big( n_1, \frac{N}{2} \Big) \frac{N}{2} \Big( \frac{N}{2}-n_1 \Big) 
= \frac{n_1 N^2}{4},
$$
where in the last step we used the assumption $n_1 \le N/2$ and simplified the
expression. 
On the other hand, definition \eqref{eq: estimated capacity} gives:
\begin{align*} 
\widehat{C}(n_1,\ldots,n_L,1) 
  &\le n_1 \sum_{k=1}^{L-1} n_k n_{k+1}
  	\quad \text{(since $\min(n_1,\ldots,n_k) \le n_1$)} \\
  &\le \frac{1}{4} n_1 (n_1+\cdots+n_L)^2
  	\quad \text{(using Lemma~\ref{lem: quadratic form})} \\
  &=\frac{n_1 N^2}{4}.
\end{align*}
Comparing the two bounds completes the proof.
\end{proof}

\begin{lemma}[Large input layer]		\label{lem: large input layer}
  If $n_1 \ge N/2$, then: 
  $$
  \widehat{C}(n_1,\ldots,n_L) 
  \le n_1^2(n_2+\cdots+n_L)
  = \widehat{C}(n_1, n_2+\cdots+n_L).
  $$    
\end{lemma}

\begin{proof}
The assumption that $n_1 \ge N/2 = (n_1+\cdots+n_L)/2$ implies that
$n_1 \ge n_2+\cdots+n_L$, and in particular we have $n_1 \ge n_k$ for all $k \ge 1$. 
Therefore, by the definition of the estimated capacity \eqref{eq: estimated capacity},
we have: 
\begin{align*} 
\widehat{C}(n_1,\ldots,n_L) 
  &\le n_1^2 \sum_{k=1}^{L-1} n_{k+1}
  	\quad \text{(since $\min(n_1,\ldots,n_k) \le n_1$ and $n_k \le n_1$)}  \\
  &= n_1^2(n_2+\cdots+n_L).
\end{align*}
On the other hand, by the definition of the estimated capacity \eqref{eq: estimated capacity}, we also have:
$$
\widehat{C}(n_1, n_2+\cdots+n_L)
  = n_1^2(n_2+\cdots+n_L).
$$
This completes the proof.
\end{proof}

\begin{proof}[Proof of Theorem~\ref{thm: max capacity given nodes and input}]
Consider the case $n_1 \le N/2$ first. 
Because the capacity is bounded by the estimated capacity (Proposition~\ref{prop: multi-channel upper}),  
Lemma~\ref{lem: small input layer} gives:
$$
C(n_1,\ldots,n_L) 
\le \widehat{C}(n_1,\ldots,n_L) 
\le \frac{n_1 N^2}{4}
= \widehat{C} \Big( n_1, \frac{N}{2}, \frac{N}{2}-n_1 \Big). 
$$
Furthermore, by Theorem~\ref{thm: three layers multi-channel},
the estimated capacity is equivalent to the actual capacity, i.e. 
$$
C \Big( n_1, \frac{N}{2}, \frac{N}{2}-n_1 \Big)
\asymp \widehat{C} \Big( n_1, \frac{N}{2}, \frac{N}{2}-n_1 \Big).
$$
This yields the first part of the conclusion.

We can argue similarly in the case $n_1 \ge N/2$. Indeed, using Lemma~\ref{lem: large input layer}, 
we obtain:
$$
C(n_1,\ldots,n_L) 
\le \widehat{C}(n_1,\ldots,n_L) 
\le n_1^2(n_2+\cdots+n_L)
= \widehat{C}(n_1, n_2+\cdots+n_L).
$$    
Finally, Theorem~\ref{thm: Zuev} yields:
$$
C(n_1, n_2+\cdots+n_L) = C(H^{n_1}) (n_2+\cdots+n_L) 
\approx n_1^2 (n_2+\cdots+n_L) = \widehat{C}(n_1, n_2+\cdots+n_L).
$$
This completes the proof of the theorem.
\end{proof}

\begin{remark}[Optimal single-output newtorks]
  One can state similar results for single-output architectures $A(n_1, . . . , n_L, 1)$, 
  because their capacities are equivalent to the capacities of $A(n_1,\ldots,n_L)$ 
  (Corollary~\ref{cor: effect of output node}). We skip the details.
\end{remark}

\subsection{{Minimizing capacity}}
In the theorems above we have maximized the capacity. It is also possible to minimize the capacity and here too, everything else being equal, we find that capacity tends to be minimized by deep architectures. For example, we have the theorem:

\begin{theorem}		\label{thm: min capacity given weights/neurons and input}
Consider the set of architectures of the form $A(n_1,\ldots, n_L, 1)$ with $L \ge 2$. Assume that $n_1$ is fixed, and that either the number of connections $W$ or the number of nodes $N$ is fixed. In either case, the capacity
is minimized by the deepest possible architecture with $n_2=n_3= \cdots =n_L=1$.
\end{theorem}

\begin{proof}
By definition, we must have at least one unit in each hidden layer, and each layer must be fully connected to the following layer. By Theorem~\ref{thm: Zuev}, the first hidden layer contributes at least 
$n_1^2(1+o(1))$ to the capacity and this number is minimized by having a single unit in the first hidden layer. If we stack layers of size 1 above this layer, the capacity remains unchanged and thus is minimized. Note that in this case the number of layers $L$ is dictated by the value of $W$ or $N$. Thus,
the minimal capacity is attained by the architecture $A(n_1,1,\ldots,1)$.
\end{proof}

\section{Structural regularization}	\label{s: structural regularization}

Some have attributed the power of deep networks to the ability of being able to compute more functions. The results of the previous section, summarized in Corollary~\ref{cor: optimization intro}, show that this cannot be the case as the opposite is true: everything else being equal, capacity tends to be maximized by {\em shallow} networks. However the functions computable by shallow and deep networks are different. For example,
R.~Eldan and O.~Shamir \cite{eldan2016power} found that a three-layer network 
with moderate-sized hidden layers is able to compute certain functions 
that a two-layer network is unable to compute, unless its hidden layer has exponential size. 
Thus, the emerging picture is that {\em deeper networks with the same number of nodes compute fewer but  more sophisticated functions}. 
This lead to the notion of {\it structural regularization}. 

It has often been noted that deep networks have a tendency to avoid overfitting, even when the 
size of the training set is small compared to the number of parameters $W$ (\cite{zhang2016understanding} and references therein). 
Some of this affect has been attributed to the regularizing properties of 
the main learning algorithm--stochastic gradient descent, and its inherent tendency to converge towards critical points with relatively broad basins of attraction
(e.g. \cite{zhu2018anisotropic} and references therein).
However, the results presented here show that there is a major regularization associated with deep architectures that is purely structural and   independent of the learning algorithm: compared to shallow networks, deep networks compute fewer functions, but these functions tend to be ``smoother and more sophisticated''.
The functions we see in practice are a tiny fraction of the universe of all possible functions, but they are the most interesting ones. And deep networks are able to ``focus'' on them. 
To see this more formally we can look at the behavior of various architectures on real-valued inputs.
The situation is very different in the one-dimensional case, versus all other higher dimensional cases, as shown in the following results.
In the one-dimensional case, the behavior of the architecture depends exclusively on the size of the first hidden layer and adding hidden layers does not increase the space of functions that can be implemented. 

\begin{proposition}
  The set $T(1,n_2,\ldots,n_L,1)$ consists of all piecewise-constant functions $f: \R \to \{0,1\}$ 
  with at most $n_2$ points of discontinuity. In particular, this class is 
  determined entirely by $n_2$ alone.
\end{proposition}

\begin{proof}
The first hidden layer, through the $n_2$ biases, creates $n_2$ potential points of discontinuity. Since there is a single output, every function $f \in T(1,n_2,\ldots,n_L,1)$ must be constant, and equal to 0 or 1 on each of the corresponding $n_ 2+1$ regions.  It is possible to select the units in the hidden layer such that the leftmost region is coded by the vector $(0,0, \ldots,0)$ in the hidden layer,
the second leftmost regions is coded by the vector 
$(1,0, \ldots,0)$, the third leftmost region is coded by the vector
$(1,1, \ldots,0)$, and so forth until the rightmost region which is coded by $(1,1, \ldots 1)$. The corresponding matrix, augmented with the vector $1,1,\ldots 1)$ to account for the bias has full rank $n_2+1$. Therefore, by selecting the proper weights and biases, any value 0 or 1 can be assigned by the architecture to each one of the regions.
\end{proof}

\begin{proposition}	\label{prop: 1}	
The set of functions $T(n,m,1)$ is characterized first by a splitting of $\R^n$ into at most
$L(m,n)=\sum_{k=0}^n {m \choose k}$ regions, each one of which produces a constant binary vector in the hidden layer, and then the assignment of a 0 or 1 output to each region which can be achieved in at most:
$$ 2 \sum_{k=0}^m {{L(m,n)-1} \choose k} \leq m\log_2 L(m,n) $$ 
different ways (the last inequality assumes $m \geq 4$).
\end{proposition}

\begin{proof}
The proof is easily obtained by using Theorem~\ref{thm: hyperplane partition} to obtain the number $L(m,n)$ of regions, noting that each region is mapped into a fixed vector in the hidden layer, and them applying Lemma~\ref{lem: capacity elementary upper} with $\vert S \vert=L(m,n)$.
\end{proof}

As an example, consider the class of $A(2,m,1)$ architectures. The hidden layer gives rise to $m$ affine lines that partition the input space $\R^2$ into
$(m^2+m+2)/2 \approx m^2/2$ regions. The number of possible binary assignments to these regions scales like $2^{m^2/2}$. While in principle the output unit could have capacity $m^2$ and thus be able to handle all these assignments, in reality is capacity is reduced because only $m^2/2$ vectors, out of all possible $2^m$ vectors, are seen in the hidden layer. Thus the capacity of the output unit is considerably reduced to be at most: $m \log_2(m^2/)2$, using the standard upperbound on the capacity of sets.

The same approach can be applied to deep architectures.

\begin{proposition}		\label{prop: 2}	
The set of functions $T(n_1,n_2,\ldots, n_L,1)$ is characterized first by a splitting of $\R^{n_1}$ into $L(n_2,n_1)=\sum_{k=0}^{n_1} {n_2 \choose k}$ regions. Each one of these regions is mapped to a fixed binary vector in the first hidden layer, creating a set $S \subset \R^{n_2}$. The capacity of the number of functions that can be computed by the upper part of the architecture is given by 
$C(S,n_2, n_3, \ldots,n_L,1)$ and can be bounded using the results in Sections~\ref{s: upper bounds} and \ref{s: restricted capacity lower}.
\end{proposition}

In short, the emerging intuitive picture is that the first hidden layer determines the number of regions into which the input space is fractured. The overall function is constant in each one of these regions, irrsepective of its depth. The larger the first hidden layer is, the greater the number of such regions. A network with a single, non-exponential hidden layer, has limited power in terms of assigning values to these regions. A deep network with the same number of parameters and hence a smaller first hidden layer will fracture the input in less regions and thus its output will have fewer regions of discontinuity. On the other hand the deep network will be able to compute more complex assignments to these regions.  

\section{{Polynomial threshold functions}}
\label{s: polynomial}


In search of more accurate models for biological neurons, or more powerful computational models, one may replace the linear activation with a polynomial activation of degree $d$ in the input variables, usually using a lower degree polynomial. Recently, we were able to develop a theory for the capacity $C_d(n,1)$ of a single polynomial threshold gate \cite{baldi2018boolean} with $n$ inputs, generalizing Zuev's result (Theorem~\ref{thm: Zuev}) for all $d \geq 1$ and showing that $C_d(n,1)= [n^{d+1}/d!] (1+o(1))$. The set and network capacity results presented here should be extended to feedforward networks of polynomial threshold functions.
We present a first step in that direction beyond what is already in \cite{baldi2018boolean}. 
First we have the following theorem which generalizes Lemma~\ref{lem: capacity elementary upper}. 

\begin{theorem}[Polynomial set capacity]			\label{thm: polynomial set capacity}
  Consider a finite subset $S \subset \R^n$, where $n>1$.
  Then, for any degree $1 < d \le n$, we have:
  $$
  C_d(S) \le \log_2 \Big ( 2 \sum_{k=0}^{M(n,d)-1}{\vert S\vert -1 \choose k} \Big ) \le \left( M(n,d)-1\right )\log_2 \vert S\vert \le
    \Big( \frac{2en}{d} \Big)^d \log_2 \vert S \vert.
  $$ 
  where:
$$
M(n,d)=\sum_{k=0}^{d}{{n+k-1} \choose k} 
\le \sum_{k=0}^{d}{{2n} \choose k}
\le \Big( \frac{2en}{d} \Big)^d.
$$  
\end{theorem}

\begin{proof}
First, it is easy to see that the number of coefficients of a polynomial of degree $d$ in $n$ variables $x_1, \ldots,x_n$ is given by $M(n,d)$, including the constant term (bias).
A vector $x \in \R^n$ can be canonically and injectively mapped into a 
vector $f(x) \in \R^{M(n,d)-1}$ whose components are the various monomials. 
 Using this mapping, we can represent any polynomial $p(x)$ of degree $d$ over $\R^n$
as a linear affine function over $f(x)$. 
And vice versa, any linear affine function over $f(x)$ is a polynomial of degree $d$ over $x$.
For example, if $d=2$, the vector $x=(x_1,x_2) \in \R^2$ is canonically mapped 
to the vector $f(x)=(x_1,x_2, x_1x_2, x_1^2,x_2^2) \in \R^5$.
Any polynomial $p(x) = a_0 + a_1 x_1 + a_2 x_2 + a_{12} x_1 x_2 + a_{11} x_1^2 + a_{22} x_2^2$
over $\R^2$ is clearly an affine function of $f(x)$, and vice versa.
Therefore: 
$$
C_d(S) = C_1(f(S))=C(f(S)).
$$
We complete the proof by applying Lemma~\ref{lem: capacity elementary upper} for the set 
$f(S) \subset \R^{M(n,d)-1}$, noting that $f(S)$ has the same cardinality as $S$ since $f$ is injective. Note that when $n \geq 2$ and $d \geq 2$, $M(n,d) \geq 6$ which is required for the application of the second part of Lemma~\ref{lem: capacity elementary upper}.
\end{proof}
Note that if we apply Theorem~\ref{thm: polynomial set capacity} to $S=H^n$, we get: 
$$ C_d(H^n) \leq n \Big( \frac{2en}{d} \Big)^d,
$$
which is somewhat weaker asymptotically than the result in 
\cite{baldi2018boolean} giving: 
$$C_d(H^n)=C_d(n,1)= \frac{ n^{d+1}}{d!}(1+o(1)).$$ Note also that the general lower bound:
$ 1+\log_2 \vert S \vert \leq  C_d(S)$, and its improved version when $S$ is a subset of the Boolean cube:
$ \log_2^2 \vert S\vert/16 \leq  C_d(S)$ are trivially satisfied.

Using Theorem~\ref{thm: polynomial set capacity}, we can now prove the first result for fully-connected feedforward networks of polynomial threshold gates of degree $d$ with a single hidden layer:

\begin{theorem}[Polynomial capacity of a single-hidden-layer network]
\label{thm: polynomial network capacity}
Consider a feedforward, fully-connected, feedforward network $A(n,m,1)$ of polynomial threshold gates of fixed degree $d$. If $n \to \infty$  and $\log_2m = o(n)$ then:

$$m \frac{n^{d+1}}{d!} (1+o(1)) \leq C_d(n,m,1) \leq m \frac{n^{d+1}}{d!} +
\min \Big [ \frac{m^{d+1}}{d!},n\Big( \frac{2em}{d} \Big)^d \Big ].$$
\end{theorem}

\begin{proof}
The proof follows somewhat what happens in the case $d=1$, using the result in \cite{baldi2018boolean} for single units.
The capacity of the hidden layer alone is given by:
$$m \frac{n^{d+1}}{d!} (1+o(1))$$ and it is easy to check that the multiplexing technique can equally be applied to this case. This immediately yields the lower bound:
$$ m \frac{n^{d+1}}{d!} (1+o(1))    \leq C_d(n,m,1). $$
For the upper bound, the total capacity is bounded by the sum of the capacity of all the units.
The capacity of each unit in the hidden layers is bounded by $n^{d+1}/d!$. The capacity of the output unit is bounded by $m^{d+1}/d!$. For the output unit, 
using Theorem~\ref{thm: polynomial set capacity}, its capacity is also bounded by:
$$n\Big( \frac{2em}{d} \Big)^d  $$
if $m>n$, and by:
$$m\Big( \frac{2nm}{d} \Big)^d $$
if $m \leq n$. This completes the proof.
\end{proof}

Note that in the case where $m^d/n^{d+1}=o(1)$ we get: 
$$C_d(n,m,1) = m \frac{n^{d+1}}{d!} (1+o(1)).$$

{The work presented here naturally leads to several open research questions. We briefly mention a few.}

\section{{Open questions}}
\label{s: open}

\subsection{{Polynomial threshold functions.}}

{The initial results given above on polynomial threshold functions need to be extended in several directions. We leave the tightening of the capacity result in Theorem~\ref{thm: polynomial network capacity}  and its extensions to networks with multiple hidden layers of polynomial threshold functions of degree $d$ for future work. The same is also true for any extensions of the lower bound on set capacity in Theorem~\ref{thm: BTF on general S}  
to polynomial threshold functions of degree $d$.}

\subsection{Asymptotic tightness.}
Theorem~\ref{thm: main} presents a capacity formula that is accurate within an absolute constant factor. 
Is this formula asymptotically tight? 
In other words, is it true that: 
$$
C(n_1,\ldots,n_L) 
= (1+o(1)) \sum_{k=1}^{L-1} \min(n_1,\ldots,n_k) n_k n_{k+1}
$$
as the number of nodes $n_k$ for some (or all) layers $k$ increases to infinity?

{
The upper bound in Theorem~\ref{thm: main} is indeed tight (Proposition~\ref{prop: multi-channel upper}), but we only know that lower bound is asymptotically tight for networks with a single hidden layer (Corollary~\ref{cor: two layers}).
And even beyond that, can the $1+(o(1)$ term be further improved, in the same way that for single neurons the result in \cite{kahn1995probability} refines the capacity estimate in  
\cite{zuev1989asymptotics}?}

\subsection{Restricted capacity.}
In Sections~\ref{s: upper bounds} and \ref{s: restricted capacity lower} 
we estimated the restricted capacity $C(S, n_1,n_2,\ldots,n_L,1)$ 
for a general input set $S \subset H^n$. 
Our lower bound (Proposition~\ref{prop: restricted vs unrestricted}) is tight within a logarithmic factor. 
Can this factor be removed? 

To do so, one may try to follow the proof of the lower bound in Theorem~\ref{thm: main} 
and use Theorem~\ref{thm: BTF on general S} instead of the estimate $C(H^n) \asymp n^2$
in this argument. But this reasoning meets an obstacle: we do not know
a version of the Enrichment Theorem~\ref{thm: enrichment} for a general set $S$. 

Thus, an important related problem is to generalize the Enrichment Theorem~\ref{thm: enrichment} 
for a general set $S \subset H^n$. Is it true that there exists an injective map $F : S \to H^m$
all of whose components are threshold functions and such that:
$$
C(S) \asymp m \log |S| ?
$$

\subsection{Other transfer functions}

This paper focused exclusively on networks with the threshold (Heaviside) transfer function
$h= \one_{\{t \ge 0\}}$.
One may wonder if our results hold true for other transfer functions, 
such as the ReLU function 
$\max(0,t)$, or the sigmoidal transfer functions (e.g. logistic 
$1/(1+e^{-t})$ or $\tanh t$) that are commonly used
in neural networks. Suppose some general transfer function $f(t)$ is applied at each hidden layer. As long as the inputs are from a finite set and the threshold function is applied at the output nodes, the set of functions that can be implemented by the architecture remains finite and the same definition of capacity can be applied without the need for any adjustments. How many functions can the network compute? 

Our preliminary analysis, left for further investigations, suggests that
the capacity might not depend too much on the shape of the transfer function $f$, 
provided that $f$ is a piecewise-constant function that consists of more than one piece, but less 
than an exponential number of pieces. Thus, the capacity formula in Theorem~\ref{thm: main} 
might be universal for the class of networks with piecewise-constant transfer functions. 

Beyond piecewise-constant, we can consider piecewise-linear transfer functions, such as ReLU functions. We have recently shown that the capacity of a single unit
is increased by a ReLU transfer function, but only marginally as it remains equal to $n^2(1+o(1))$ \cite{baldi2018neuronal}.
We conjecture that, in essence, the same remains true for multilayer networks with ReLU transfer functions in the hidden layers. In light of the known results 
(\cite{bartlett1999almost, bartlett2017nearly}) that show that the VC-dimension of ReLU neural networks 
may grow {\em super-linearly} with the depth $L$, it is interesting to find out 
if the capacity of the networks (e.g. with equal sizes of layers) grows super-linearly as well.

The definition of capacity used here relies here on a class of functions or hypotheses $T$ that is finite. However, the definition could clearly be extended to networks that can compute an infinite number of functions, for instance using real-valued inputs and continuous transfer functions everywhere, including in the output layer. For this purpose, one could still define capacity by $C= \log_2 \vert T \vert$ but defining
 $\vert T \vert $ in a broader measure theoretic sense (e.g. volume).

\subsection{Other connectivity models}

Finally, it is possible to consider several other connectivity models,
either by having more sparse connections or by constraining the values of the connections. A first example is to consider synaptic weights that are constrained to belong to a finite set, for instance $\{-1,1\}$ (binary synapses). It is easy to see that the capacity of a binary synapse linear threshold neuron is exactly equal to the number of inputs $n$
\cite{baldi2018neuronal}, but the extension to multiple layers has not been studied. Likewise, it is possible to consider the case where all the incoming, or outgoing, synaptic weights must have the same sign (e.g. purely excitatory or purely inhibitory neurons). We have shown that, for a single linear threshold unit, if all the incoming weights are positive the capacity of the unit is marginally decreased but remains in the same class and equal to $n^2(1+o(1))$. Again the extension to multiple layers has not been studied. Finally, there is the case of local connectivity,
where the indegree or outdegree of a neuron may be restricted. How can the theory be extended to some of those cases? Consider a typical convolutional neural network for computer vision applications. In this case, a convolutional layer may comprise an array of $n$ neurons, where each neuron has an identical set of $m \times m$ incoming weights, the so-called weight-sharing approach. Because the weights are tied, the entire layer can implement only one set of weights, and thus only one function. If the neurons are modeled as linear threshold gates, the capacity of the entire layer can be estimated, and is equal to $m^4(1+(o(1))$. Thus, in short, under the right assumptions the methods presented here can readily be applied to convolutional neural networks.
All the previous examples, assumed feedfoward patterns of connections. 
Extensions to recurrent networks, other than the fully-connected case, have not been investigated.

\section{{Conclusion}}
\label{s: conclusion}

{
In the 1940s, McCulloch and Pitts \cite{mcculloch:43} and others introduced a simple neuronal model, whereby a neuron processes information by first computing an activation $A$ and then an output $O=f(A)$. The activation $A$ is typically a weighted linear, or polynomial, function of the inputs. The transfer function $f$ is typically a non-linear function, such as a threshold function, ReLU (rectified linear unit) function or more generally a piecewise linear function, or sigmoidal function (e.g. logistic, $\tanh$).}

{Networks of McCulloch and Pitts model neurons are important for at least four fundamental reasons. First, 
although far simpler than biological neurons in their processing details, these simplified neural models have proven over and over to be useful to better understand biological networks (e.g. \cite{zipser1988back,olshausen1996emergence,yamins2016using}).
Second, these models are also used to guide the development of new, power efficient, neuromorphic chips \cite{neftci_event-driven_2017}.
Third, these  neural network models are widely used today in all kinds of AI/deep learning applications with impressive results, often matching or exceeding human capabilities in specific tasks across the gamut of  applications, from games to biomedicine (e.g.\cite{schmidhuber2015deep,silver2017mastering1,
baldireview2018}). And finally, from a foundational standpoint, they are the dominant, and perhaps simplest, available analytical model for studying the neural style of storing information.} 

{Indeed, in the standard description of McCulloch and Pitts neurons given above, the emphasis is placed on the processing aspect of these models, the input-output relationship. However, equal or even more emphasis should be given to the storage aspect of this neural model. Information about the world, e.g. in the form of ``training sets'', is stored in a distributed ``holographic'' way in the synaptic weights (i.e. the coefficients of the activation function), through a learning process.
This is significant because
to achieve intelligent behavior, information processing systems must be able to learn and store information. To store information, there are two completely different approaches:
(1) the Turing tape model, 
where information is stored at well organized, discrete, addresses of a physical substrate--this is the style used in all living systems at the cellular level (DNA), and in all our digital devices, from cell phones to supercomputers; (2) the neural model, where information is stored ``holographically'' in neural networks across large numbers of synapses--this is the style believed to be used by the brain, and simulated in our neural network--deep learning--technology. While the Turing style is relatively well understood, the neural style is not. }

{
In this paper, we set out to study the most fundamental property of the neural style of storage, namely how many bits can be stored in a given neural architecture. To address this question, we first had to introduce the notion of cardinal capacity, the logarithm base two of the number of different functions a given architecture can compute. Remarkably, for neural architectures, the cardinal capacity is equal to the total number of bits that can be stored in a given architecture, or the number of bits that can be ``communicated'' from the outside world to the architecture by the learning process. 
We then estimated the capacity of feedforward neural architectures of arbitrary depth, under a relatively mild set of assumptions on the connectivity and the transfer functions of these architectures. The capacity is typically a cubic polynomial in the sizes of the layers. For fully connected, feedforward, architectures it is essentially given by:
$C(n_1,\ldots, n_L)\approx \sum_{k=1}^{L-1} \min(n_1,\ldots,n_k)n_kn_{k+1}$.
As a side note, the capacity of fully connected recurrent networks can also be estimated \cite{baldi2018neuronal}, essentially by unfolding them in time and computing the capacity of the underlying feedforward network.
In addition, we have improved the bounds on the capacity of sets, analyzed the extremal properties of the capacity and the structural regularization effects of deep architectures,
and began to extend the theory of capacity to polynomial threshold functions. In addition, we have briefly surveyed several open questions in this area.}

{Finally, although this falls beyond the scope of this paper, 
the capacity is a fundamental quantity that can be related to other measures of complexity and generalization including the VC dimension, the growth function, 
the Rademacher and Gaussian complexity, the metric entropy, and
the minimum description length (MDL). 
For example, if the function $h$ to be learnt as MDL $D$, and the neural architecture being used as capacity $C < D$ then it is easy to see that: (1) $h$ cannot be learnt without errors; and (2) the number $E$ of errors made by the best approximating function implementable by the architecture must satisfy $E > (D - C)/N$. Another example of connection is the connection to the VC dimension that was used in the bounds in Section~\ref{s: restricted capacity lower}. These connections will be described more systematically elsewhere. }

\section*{Appendix: Examples}
{In this appendix, we apply the main result to a few basic architectures, assuming the layers are large so that the asymptotic regimes can be applied. }

{
First, consider a deep architecture $A(n_1,\ldots,n_L)$ which is expansive, 
where expansive is defined by the property: $n_1\leq n_2\leq \ldots \leq n_{L-1}$. Then, using the main result:
$$C(A) \approx n_1 \sum_{k=1}^{L-1}n_kn_{k+1} \approx n_1W.$$}

{Second, consider a deep architecture $A(n_1,\ldots,n_L)$ which is compressive, where compressive is defined by the property: 
$n_1 \geq n_2\geq \ldots \geq n_{L-1}$. Then, using the main result:
$$C(A) \approx \sum_{k=1}^{L-1}n_k^2n_{k+1}.$$}

\vspace{-0.3cm}
\begin{figure}[htp]			
  \centering 
    \includegraphics[width=0.7\textheight]{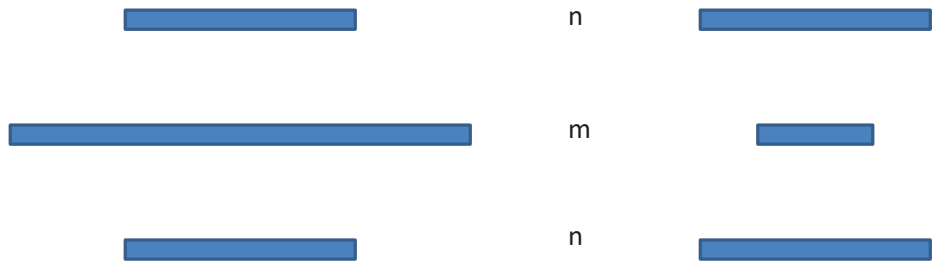} 
    \vspace{-3.0cm}
    \caption{Expansive (left) and compressive (right) autoencoder architectures $A(n,m,n)$.}
  \label{fig:autoencoder}
\end{figure}

{Third, we consider autoencoder architectures
$A(n,m,n)$ with a single hidden layer (Figure~\ref{fig:autoencoder}).}
Clearly $W=2nm$. For the capacity, there are two cases depending on whether the autoencoder is expansive or compressive.
In the compressive case ($m <n$), $C(n,m,n)\approx mn^2+ m^2n=mn(n+m)\approx mn^2=nW/2$. If we let
$m=n^{1-\nu}$ for $0< \nu <1$, then:
$$ C(n,m,n)\approx n^{3-\nu}.$$
In the expansive case ($m>n$), $C(n,m,n)\approx mn^2+ mn^2=2mn^2=nW$. If we let
$m=n^{1+\nu}$ for $0< \nu $, then:
$$ C(n,m,n)\approx 2 n^{3+\nu}.$$
 \vspace{-0.5cm}
\begin{figure}[htp]			
  \centering 
    \includegraphics[width=0.7\textheight]{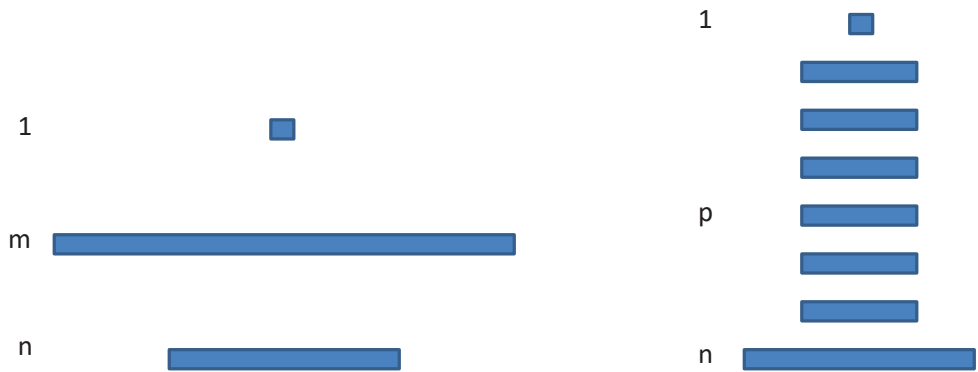} 
    \vspace{-3.0cm}
    \caption{Shallow (left) $A(n,m,1)$ architecture and deep (right)$A(n,p,p,...,p,1)$ architecture for classification.}
  \label{fig:classification}
\end{figure} 
 
Finally, we contrast a shallow $A(n,m,1)$ and a deep 
$A(n,p,\ldots,p,1)$ classification architectures (Figure~\ref{fig:classification}). For the shallow architecture, 
we have:
$$ W \approx nm \quad {\rm and} \quad  C \approx mn^2.$$
For the deep architecture, we have:
$$W \approx np+Lp^2 \quad {\rm and} \quad  C \approx pn^2 + L p^3$$
if $p\leq n$, and: 
$$W \approx np+Lp^2 \quad {\rm and} \quad  C \approx pn^2 + L np^2$$
if $p>n$. Here $L$ is a parameter that represents the depth--the entire architecture has $L+2$ layers, not counting the single-unit output layer.
Consider, for instance, the expansive case where $m\geq n$ and $p \geq n$.
Then both architectures satisfy: $C \approx n W$ and will have roughly the same capacity when  they have roughly the same number of parameters. 
If we let $m=n^{1+\alpha}$ ($\alpha \geq 0$), $p=n^{1+\beta}$ ($\beta \geq 0$), and $L=n^{\gamma}$ ($\gamma \geq 0$) then for the architectures to have approximately the same number of parameters (and thus approximately the same capacity), one must have: $2\beta +\gamma =\alpha$. The other cases can be analyzed similarly. 
 

\bibliography{capacity,baldi,nn,math}{}
\bibliographystyle{plain}

\end{document}